\documentclass{article}

    \usepackage[preprint]{neurips_2021}

\usepackage[utf8]{inputenc} 
\usepackage[T1]{fontenc}    
\usepackage{hyperref}       
\usepackage{url}            
\usepackage{booktabs}       
\usepackage{amsfonts}       
\usepackage{nicefrac}       
\usepackage{microtype}      
\usepackage{xcolor}         

\usepackage{algorithm, algorithmic}
\usepackage{natbib}
\usepackage{graphicx}
\usepackage{subfigure}
\usepackage{multirow}
\usepackage{amsmath, bm, amsthm}
\usepackage{wrapfig}
\usepackage{xfrac}
\usepackage{sidecap}
\usepackage{adjustbox}
\usepackage{caption}
\usepackage{bbm}

\newcommand{\id}{{\mathrm{id}}}

\newcommand{\orb}{{\mathrm{orb}}}

\newcommand{\mean}{\mathbb{E}}
\newcommand{\var}{{\rm I\kern-.3em D}}

\newcommand{\cond}{\,|\,}
\newcommand{\Normal}{\mathcal{N}}

\newcommand{\norm}[1]{\left\lVert#1\right\rVert}
\newcommand{\eps}{\varepsilon}

\DeclareMathOperator*{\infimum}{inf}

\newtheorem{theorem}{Theorem}
\newtheorem*{theorem*}{Theorem}

\newtheorem{proposition}{Proposition}
\newtheorem*{proposition*}{Proposition}
\newtheorem{corollary}{Corollary}
\DeclareMathSymbol{\shortminus}{\mathbin}{AMSa}{"39}

\usepackage{color}

\title{Orbital MCMC}

\author{%
  Kirill Neklyudov \thanks{Parts of this work was done while the author was at the Higher School of Economics, Moscow, Samsung-HSE laboratory.} \\
  University of Amsterdam\\
  \texttt{k.necludov@gmail.com} \\
   \And
   Max Welling \\
   University of Amsterdam, CIFAR\\
  \texttt{m.welling@uva.nl} \\
}

\begin{document}

\maketitle

\begin{abstract}
Markov Chain Monte Carlo (MCMC) algorithms ubiquitously employ complex deterministic transformations to generate proposal points that are then filtered by the Metropolis-Hastings-Green (MHG) test.  However, the condition of the target measure invariance puts restrictions on the design of these transformations. In this paper, we first derive the acceptance test for the stochastic Markov kernel considering arbitrary deterministic maps as proposal generators. When applied to the transformations with orbits of period two (involutions), the test reduces to the MHG test. Based on the derived test we propose two practical algorithms: one operates by constructing periodic orbits from any diffeomorphism, another on contractions of the state space (such as optimization trajectories). Finally, we perform an empirical study demonstrating the practical advantages of both kernels.
\end{abstract}

\section{Introduction}
\begin{figure}[h]
    \centering
    \includegraphics[width=0.9\textwidth]{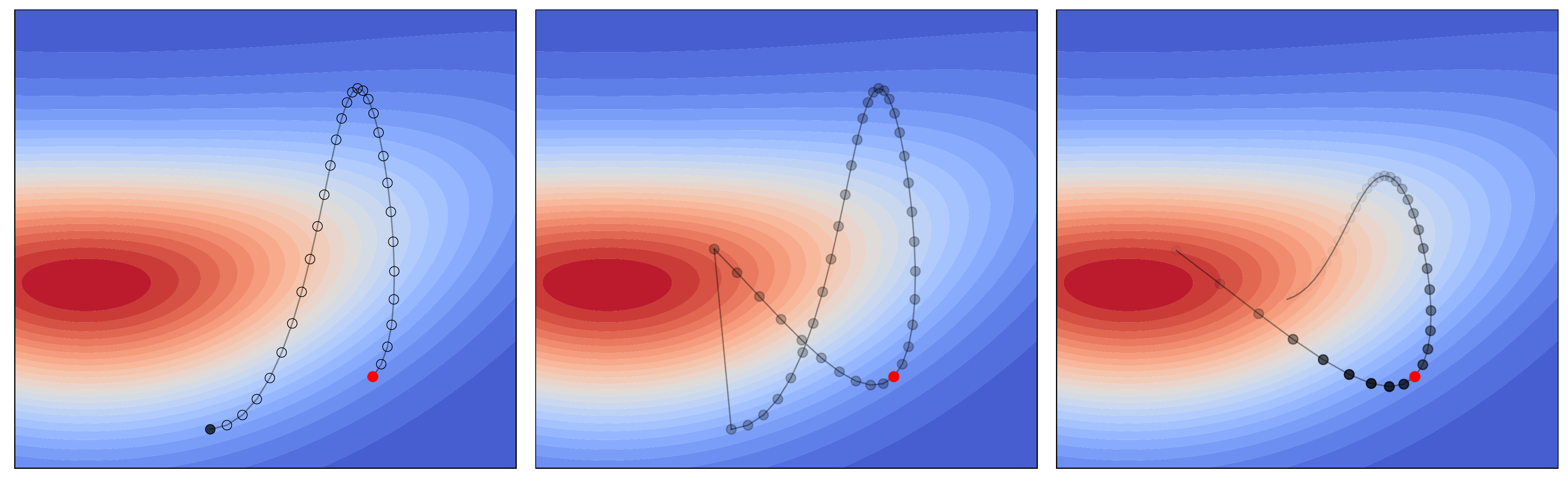}
    \caption{The illustration of the proposed algorithms. The transparency of a point corresponds to its weight. Red dots correspond to the initial states. From left to right: HMC (stochastically accepts only the last state), Orbital kernel on a periodic orbit (accepts all the states weighted), Orbital kernel on an infinite orbit (accepts the states within a certain region).}
    \label{fig:intro}
\end{figure}
MCMC is an ubiquitous analysis tool across many different scientific fields such as statistics, bioinformatics, physics, chemistry, machine learning, etc.
Generation of the proposal samples for continuous state spaces usually includes sophisticated deterministic transitions. 
The most popular example of such transitions is the flow of Hamiltonian dynamics \citep{duane1987hybrid, hoffman2014no}.
More recently, the combination of neural models and conventional MCMC algorithms has attracted the community's attention \citep{song2017nice, hoffman2019neutra},
and such hybrid models already found their application in the problems of modern physics \citep{kanwar2020equivariant}.
These developments motivate further studies of the possible benefits of deterministic functions in MCMC methods.

One possible way to introduce a deterministic transition into the MCMC kernel is to consider an involutive function inside a stochastic transition kernel as described in \citep{neklyudov2020involutive, spanbauer2020deep} (we recap Involutive MCMC in Section \ref{sec:imcmc}).
One of the practical benefits of this approach is that it allows one to learn the kernel as a neural model improving its mixing properties \citep{spanbauer2020deep}.

Our paper extends prior works by deriving novel transition kernels that do not rely on involutions.
We start by deriving the novel acceptance test that is applicable for any diffeomorphism and operates on its orbits.
However, in its general form the test is infeasible to compute requiring the evaluation of the infimum across the orbit.
To obtain practical algorithms we consider two special cases. 
Our first algorithm operates on periodic orbits, and we demonstrate how one can easily design such orbits by introducing auxiliary variables (Section \ref{sec:orbital_kernel}).
The second algorithm operates on infinite orbits of optimization trajectories.
We demonstrate how such orbits could be truncated preserving the target measure (see section \ref{sec:diffusing_orbital_kernel}).
Finally, we provide an empirical study of these algorithms and discuss their advantages and disadvantages (Section \ref{sec:practice}).

\section{Background}
\label{sec:background}
\subsection{Mean ergodic theorem}
The core idea of the MCMC framework is the mean ergodic theorem, namely, its adaptation to the space of distributions.
That is, consider the $\phi$-irreducible Markov chain (see \cite{roberts2004general}) with the kernel $k(x'\cond x)$ that keeps the distribution $p$ invariant.
Denoting the single step of the chain as an operator $K: [Kq](x') = \int dx\; k(x'\cond x)q(x)$, the mean ergodic theorem says that its time average projects onto the invariant subspace:
\begin{align}
     \lim_{n\to \infty} \norm{\frac{1}{n}\sum_{i=0}^{n-1} K^i p_0 - p} = 0, \; \text{ for all }\; p_0.
     \label{eq:mean_ergodic_theorem}
\end{align}
Note, that the irreducibility condition implies the uniqueness of the invariant distribution.
In practice, the application of $K$ is not tractable, though. 
Therefore, one resorts to Monte Carlo estimates by accepting a set of samples $\{x_0,\ldots, x_{n-1}\}$, each coming with a weight $1/n$, and distributed as $x_i \sim [K^i p_0](x)$.
Thus, the MCMC field includes the design and analysis of kernels that allow for projections as in \eqref{eq:mean_ergodic_theorem}.

\subsection{Involutive MCMC}
\label{sec:imcmc}
In this section, we briefly discuss the main idea of the Involutive MCMC (iMCMC) framework \citep{neklyudov2020involutive, spanbauer2020deep}, which is the initial point of our further developments.
The iMCMC framework starts by considering the kernel
\begin{align}
    k(x'\cond x) = \delta(x'-f(x))g(x)+\delta(x'-x)(1-g(x)), \;\;
    g(x) = \min\bigg\{1, \frac{p(f(x))}{p(x)} \bigg| \frac{\partial f}{\partial x}\bigg| \bigg\},
    \label{eq:imcmc_kernel}
\end{align}
where $p(x)$ is the target density, $\delta(x)$ is the Dirac delta-function and $f(x)$ is a diffeomorphism.
The stochastic interpretation of this kernel is the following. 
Starting at the point $x$ the chain accepts $f(x)$ with the probability $g(x)$ or stays at the same point with the probability $(1-g(x))$.

The iMCMC framework then requires the kernel \eqref{eq:imcmc_kernel} to preserve the target measure: $[Kp](x') = \int dx\; k(x'\cond x)p(x) = p(x')$ yielding the following equation:
\begin{align}
    \min\bigg\{p(f^{-1}(x))\bigg|\frac{\partial f^{-1}}{\partial x}\bigg| ,p(x)\bigg\} = 
    \min\bigg\{p(x), p(f(x)) \bigg| \frac{\partial f}{\partial x} \bigg| \bigg\} \;\; \forall x.
    \label{eq:x_cond}
\end{align}
This equation allows one to bypass the difficulties of designing a measure-preserving function by choosing $f$ such that $f^{-1}(x) = f(x)$ or, equivalently $f(f(x)) = x$.
Thus, the iMCMC framework proposes a practical way to design $f$ that implies the invariance of $p(x)$ by considering the family of involutions.
However, inserting such $f$ into the transition kernel \eqref{eq:imcmc_kernel} reduces it to jump only between two points: from $x$ to $f(x)$ and then back to $f(f(x)) = f^{-1}(f(x)) = x$.

To be able to cover the support of the target distribution with involutive $f$, the iMCMC framework introduces an additional source of stochasticity into \eqref{eq:imcmc_kernel} through auxiliary variables.
That is, instead of traversing the target $p(x)$, the chain traverses the distribution $p(x,v) = p(x)p(v\cond x)$, where $p(v\cond x)$ is an auxiliary distribution that we are free to choose.
The key ingredients for choosing $p(v\cond x)$ are easy computation of its density and the ability to efficiently sample from it.
Then, interleaving the kernel \eqref{eq:imcmc_kernel} with the resampling of the auxiliary variable $v\cond x$, one can potentially reach any state $[x,v]$ of the target distribution $p(x,v)$ (see Algorithm \ref{alg:imcmc}). Samples from the marginal distribution of interest $p(x)$ can now simply be obtained by ignoring the dimensions that correspond to the variable $v$. 

\begin{wrapfigure}{r}{0.5\textwidth}
\begin{minipage}[H]{0.5\textwidth}
\vspace{-22pt}
\begin{algorithm}[H]
  \caption{Involutive MCMC}
  \begin{algorithmic}  
    \INPUT{target density $p(x)$}
    \INPUT{density $p(v\cond x)$ and a sampler from $p(v\cond x)$} 
    \INPUT{involutive $f(x,v): f(x,v) = f^{-1}(x,v)$}
    \STATE initialize $x$
    \FOR{$i = 0\ldots n$}
        \STATE sample $v \sim p(v \cond x)$
        \STATE propose $(x',v') = f(x,v)$
        \STATE $g(x,v) = \min\{1,\frac{p(x',v')}{p(x,v)} \big|\frac{\partial f(x,v)}{\partial [x,v]}\big|\}$
        \STATE $ x_{i} =
            \begin{cases} 
            x' , \text{ with probability } g(x,v)\\
            x , \;\text{ with probability } (1-g(x,v))
            \end{cases}$
        \STATE $x \gets x_{i}$
    \ENDFOR
    \OUTPUT{samples $\{x_0,\ldots, x_n\}$}
  \end{algorithmic}
  \label{alg:imcmc}
\end{algorithm}
\vspace{-20pt}
\end{minipage}
\end{wrapfigure}


It turns out that by choosing different involutions $f$ and auxiliary distributions $p(v\cond x)$ in Algorithm \ref{alg:imcmc} one can formulate a large class of MCMC algorithms as described in \citep{neklyudov2020involutive}.
However, the involutive property of the considered deterministic map $f$ enforces the iMCMC kernel (Algorithm \ref{alg:imcmc}) to be reversible: $k(x',v'\cond x,v) p(x,v) = k(x,v\cond x',v')p(x',v')$.
Moreover, the marginalized kernel on $x$: $\widehat{k}(x'\cond x) = \int dvdv'\; k(x',v'\cond x,v)p(v\cond x)$ is also reversible: $\widehat{k}(x'\cond x) p(x) = \widehat{k}(x\cond x')p(x')$, which often doesn't mix as fast as an irreversible kernel.
To describe irreversible kernels, the iMCMC framework proposes to compose several reversible kernels.

\section{Orbital kernel}
\label{sec:orbital_kernel}
\subsection{Derivation of the acceptance test}

Starting with the kernel \eqref{eq:imcmc_kernel} the iMCMC framework looks for a family of suitable deterministic maps $f$, and ends up with involutive functions (see Section \ref{sec:imcmc}).
In this work, we approach the problem of kernel design differently: given a deterministic map $f$ we develop a kernel that admits the target density as its eigenfunction, generalizing the iMCMC framework.

We start our reasoning with the kernel (which we call the escaping orbital kernel) $k(x'\cond x) = \delta(x'-f(x))g(x) + \delta(x'-x)(1-g(x))$,
where $g(x)$ is the acceptance test, and the kernel tries to make a move along the trajectory of $f(x)$ as in iMCMC.
Putting this kernel into the condition $Kp = p$: $\int dx\;k(x'\cond x)p(x) =p(x')$, we have
\begin{align}
    \int dy\;\delta(x'-y)g(f^{-1}(y))p(f^{-1}(y))\bigg|\frac{\partial f^{-1}}{\partial y}\bigg| + p(x') - \int dx\;\delta(x'-x)g(x)p(x) = p(x'),
\end{align}
where in the first term we have applied a change of variables to $y=f(x)$. This simplifies to
\begin{align}
    g(x')p(x') = g(f^{-1}(x'))p(f^{-1}(x'))\bigg|\frac{\partial f^{-1}}{\partial x}\bigg|_{x=x'}.
    \label{eq:preserving_gp_init}
\end{align}
Let's assume that $x'$ is an element of the orbit $\orb(x_0)$, which we can describe using the notation of iterated functions ($f^0(x) = \id(x) = x, f^{n+1}(x) = f(f^n(x))$) as $\orb(x_0) = \{f^k(x_0), k \in \mathbb{Z}\}$.
Then putting $x' = f^k(x_0), \; k\in\mathbb{Z}$ into \eqref{eq:preserving_gp_init} we get
\begin{align}
    \begin{split}
    \forall k \in \mathbb{Z} \;\;\;\; g(f^{k}(x_0))p(f^{k}(x_0))\bigg|\frac{\partial f^{k}}{\partial x}\bigg|_{x=x_0} = g(f^{k-1}(x_0))p(f^{k-1}(x_0))\bigg|\frac{\partial f^{k-1}}{\partial x}\bigg|_{x=x_0}.
    \end{split}
    \label{eq:g-condition}
\end{align}
Thus, instead of a single equation \eqref{eq:preserving_gp_init} at point $x'$ we actually have a system of equations generated by the recursive application of \eqref{eq:g-condition}: $\forall k \in \mathbb{Z}$
\begin{align}
    g(x_0)p(x_0) = g(f^{k}(x_0))p(f^{k}(x_0))\bigg|\frac{\partial f^{k}}{\partial x}\bigg|_{x=x_0}.
    \label{eq:preserving_gp}
\end{align}
The system tells us that the function $f$ must preserve the measure $g(x)p(x)$ on the orbit $\orb(x_0)$.
Note that if $f$ preserves the target $p(x)$ then any constant $g(x) = c$ satisfies this equation.
In this case, we can set $g(x) = 1$ and use iterated applications of $f$ for sampling from the target distribution (ensuring that the chain could densely cover the state space).
However, the design of non-trivial measure-preserving $f$ (especially with dense orbits) is infeasible in most cases.
Therefore, we can assume that $f(x)$ preserves some other density $q(x)$ on the orbit $\orb(x_0)$:
\begin{align}
    q(x_0) = q(f^{k}(x_0))\bigg|\frac{\partial f^{k}}{\partial x}\bigg|_{x=x_0} \;\; \forall k \in \mathbb{Z}.
    \label{eq:q-orbit}
\end{align}
From \eqref{eq:preserving_gp}, we also know that $f$ must preserve $g(x)p(x)$.
Thus, we can think of $g(x)$ as an importance weight that makes $g(x)p(x) \propto q(x)$. 
Interpreting $g(x)$ as a probability ($0 \leq g(x) \leq 1$), we have
\begin{align}
    0 \leq g(f^k(x_0)) = c \cdot \frac{q(f^k(x_0))}{p(f^k(x_0))} \leq 1, \;\; \forall k \in \mathbb{Z}.
    \label{eq:g_prob_condition}
\end{align}
where the constant $c$ makes sure that $g(x)\leq 1$. This is precisely achieved by $c=\infimum_{k \in \mathbb{Z}} \big\{\frac{p(f^k(x))}{q(f^k(x))}\big\}$ for any $x\in \orb(x_0)$. Finally, using the formula for $c$ and equations \eqref{eq:q-orbit}, \eqref{eq:g_prob_condition} at $x = f^k(x_0)$, the test function $g(x)$ for any $x \in \orb(x_0)$ is
\begin{align}
    \begin{split}
    g(x) = \frac{q(x)}{p(x)} \infimum_{k \in \mathbb{Z}} \bigg\{\frac{p(f^k(x))}{q(f^k(x))}\bigg\} = \infimum_{k \in \mathbb{Z}} \bigg\{\frac{p(f^k(x))}{p(x)}\bigg|\frac{\partial f^{k}}{\partial x}\bigg|\bigg\}.
    \end{split}
\end{align}
Note that $c$ is the greatest lower bound on $p(x)/q(x)$ for $x \in \orb(x_0)$, but any lower bound guarantees $g(x) \leq 1$.
The choice of the greatest lower bound is motivated by the stochastic interpretation of the chain, i.e., we want to minimize the probability of staying in the same state.
Although there is a rigorous result \citep{peskun1973optimum} for the comparison of chains via the probabilities of staying in the same state, it applies only for reversible kernels, which is not always applicable to the kernel under consideration.

Putting the test back into the kernel we have the following result.

\begin{theorem}{(Escaping orbital kernel)\\}
\label{th:o_test}
Consider a target density $p(x)$ and continuous bijective function $f$. 
The transition kernel
\begin{align*}
    \begin{split}
    k(x'\cond x) = \delta(x'-f(x)) g(x) + \delta(x'-x)(1-g(x)), \;\;
    g(x) =\infimum_{k \in \mathbb{Z}}\bigg\{\frac{p(f^k(x))}{p(x)} \bigg| \frac{\partial f^k}{\partial x}\bigg| \bigg\},
    \end{split}
    \label{eq:o_kernel}
\end{align*}
yields the operator $[Kp](x') = \int dx\; k(x'\cond x)p(x)$ that keeps the target density $p(x)$ invariant: $Kp = p$.
\end{theorem}
If we would like to accept several points from the same orbit we can substitute $f$ with $f^m$ for any integer $m$ and get the following test.
\begin{corollary}
\label{th:m_step_test}
For the the kernel $k_m(x'\cond x) = \delta(x'-f^m(x)) g_m(x) + \delta(x'-x)(1-g_m(x))$ that satisfies $K_mp=p$, the maximal acceptance probability $g_m(x)$ is
\begin{align}
    g_m(x) =\infimum_{k \in \mathbb{Z}}\bigg\{\frac{p(f^{mk}(x))}{p(x)} \bigg| \frac{\partial f^{mk}}{\partial x}\bigg| \bigg\}.
\end{align}
\end{corollary}
Once we have designed the acceptance test we can check if it is reversible.
Reversibility may hinder the mixing properties of the kernel and is usually considered as an undesirable property \citep{turitsyn2011irreversible}.
By the following proposition we see that the derived test theoretically allows us to design irreversible kernels.
\begin{proposition}{(Reversibility criterion)\\}
\label{th:reversibility}
Consider the escaping orbital kernel $k(x'\cond x) = \delta(x'-f(x))g(x) + \delta(x'-x)(1-g(x))$ that satisfies $Kp = p$, and $g(x) > 0$.
This kernel is reversible w.r.t. $p$, i.e. $k(x'\cond x)p(x) = k(x\cond x')p(x')$, if and only if $f$ is an involution, i.e., $f(x) = f^{-1}(x)$.
\end{proposition}
\begin{proof}
See Appendix \ref{app:o_kernel_reversiblity}.
\end{proof}


The bottleneck of the proposed practical scheme is the evaluation of the test. 
Indeed, it requires the evaluation of the infimum over the whole orbit, and if we don't know its analytical formula its computation becomes infeasible. 
However, as we demonstrate further, some types of orbits may simplify this evaluation.

The orbit $\orb(x_0)$ is \textit{periodic} if there is an integer $T > 0$ such that for any $x \in \orb(x_0)$, and for any $k \in \mathbb{Z}$, we have $f^{k+T}(x) = f^k(x)$. The
\textit{period} of the orbit $\orb(x_0)$ is the minimal $T > 0$ satisfying $f^{k+T}(x) = f^k(x)$.
Thus, the orbit $\orb(x_0)$ with the period $T$ can be represented as a finite set of points: $\orb(x_0) = \{f^{0}(x_0), f^{1}(x), \ldots, f^{T-1}(x)\}$, and the acceptance test for periodic orbits can be reduced to the minimum over this finite set as follows.
\begin{proposition}{(Periodic orbit)\\}
For periodic orbit $\orb(x)$ with period $T$, the acceptance test in the kernel \eqref{eq:o_kernel} (from Theorem \ref{th:o_test}) becomes
\begin{align}
    g(x) = \infimum_{k \in \mathbb{Z}}\bigg\{\frac{p(f^k(x))}{p(x)} \bigg| \frac{\partial f^k}{\partial x}\bigg| \bigg\} = \min_{k = 0,\ldots,T-1}\bigg\{\frac{p(f^k(x))}{p(x)} \bigg| \frac{\partial f^k}{\partial x}\bigg| \bigg\}.
    \label{eq:periodic_test}
\end{align}
\end{proposition}
For periodic orbits with $T=2$ (for instance, orbits of involutions), this test immediately yields the Metropolis-Hastings-Green test (or the iMCMC kernel \eqref{eq:imcmc_kernel}):
\begin{align*}
    \min_{k = 0,1}\bigg\{\frac{p(f^k(x))}{p(x)} \bigg| \frac{\partial f^k}{\partial x}\bigg| \bigg\} = \min\bigg\{1,\frac{p(f(x))}{p(x)} \bigg| \frac{\partial f}{\partial x}\bigg| \bigg\}.
\end{align*}
By considering other lower bounds in the inequalities \eqref{eq:g_prob_condition} for periodic orbits with $T=2$, one can derive other conventional tests, for instance, the tests induced by Barker's lemma \citep{barker1965monte}. Another special case appears when the points of an orbit come arbitrarily close to the initial point, which we consider in Appendix \ref{app:returning_orbit}.

Finally, we need to be able to jump between orbits to be able to cover the state space densely. We can achieve it in the same way as in Algorithm \ref{alg:imcmc}, i.e. we can introduce auxiliary variables that allow to jump to another orbit by simply resampling them.

\subsection{Accepting the whole orbit}
\label{sec:whole_orbit}

Even for periodic orbits the acceptance test \eqref{eq:periodic_test} might be inefficient because we need to evaluate the minimum across the whole orbit (and the density is not preserved).
However, the evaluation of the target density over the whole orbit allows for accepting several samples from the orbit at once.
The naive way to accept several points is to consider a linear combination of the escaping orbital kernels starting from the same point.
We discuss it in details in Appendix \ref{app:linear_combination}.

Another way to make use of the evaluated densities over the whole orbit is as follows.
Instead of switching to another orbit after the accept/reject step, we may let the kernel continue on the same orbit, and thus collect more samples.
Once again, we consider the kernel $k(x'\cond x) = \delta(x'-f(x))g(x) + \delta(x'-x)(1-g(x))$ that preserves target measure $Kp = p$. Starting from the delta-function at some initial point $p_0(x) = \delta(x-x_0)$, the chain iterates by applying the corresponding operator:
\begin{align}
    \begin{split}
    [Kp_0](x') = \int dx\; k(x'\cond x) p_0(x) = (1-g(x_0))\delta(x'-x_0) + g(x_0)\delta(x'-f(x_0)).
    \end{split}
\end{align}
Denoting the recurrence relation as $p_{t+1} = Kp_{t}$, we obtain the sum of weighted delta-functions along the orbit $\orb(x_0)$ of $f$:
\begin{align}
    p_t(x) = [K^t p_{0}](x) = \sum_{i=0}^{t} \omega_i^{t}\delta(x-f^i(x_0)),
    \label{eq:recurrence_p_t}
\end{align}
where $\omega_i^{t}$ is the weight of the delta-function at $f^i(x_0)$ after $t$ steps.
Thus, instead of considering the operator $K$ on the space of functions we consider it on the space of sequences, and  analyse the limit of the series $1/n \sum_{i=0}^{n-1}K^i$, as in the mean ergodic theorem.
The result of our analysis is as follows.

\begin{theorem}{(Convergence on a single orbit)\\}
\label{th:limiting_case_escaping}
Consider the proper escaping orbital kernel ($Kp = p$, and $g(x) > 0$) applied iteratively to $p_0(x) = \delta(x-x_0)$.
For aperiodic orbits, iterations yield
\begin{align}
    [K^t p_{0}](x) = \sum_{i=0}^{t} \omega_i^{t}\delta(x-f^i(x_0)), \;\;
    \lim_{t\to \infty} \sum_{t'=0}^t \omega_i^{t'} = \frac{1}{g(f^i(x_0))},\;\; \lim_{t\to \infty} \frac{1}{t}\sum_{t'=0}^t \omega_i^{t'} = 0.
    \label{eq:aperiodic_weights}
\end{align}
For periodic orbits with period $T$, the time average is
\begin{align}
    \lim_{t\to \infty} \frac{1}{t}\sum_{i=0}^{t-1}K^i p_0 = \sum_{i=0}^{T-1} \omega_i \delta(x-f^i(x_0)), \;\;
     \omega_i = \frac{p(f^i(x_0))\big|\frac{\partial f^{i}}{\partial x}\big|_{x=x_0}}{\sum_{j=0}^{T-1}p(f^j(x_0))\big|\frac{\partial f^{j}}{\partial x}\big|_{x=x_0}}.
    \label{eq:periodic_weights}
\end{align}
\end{theorem}
\vspace{-15pt}
\begin{proof}
See Appendix \ref{app:limiting_case}.
\end{proof}
\vspace{-10pt}

Thus, for periodic orbits, the time average of every weight converges to a very reasonable value, which is proportional to the target density at this point.
Now, using the weights from \eqref{eq:periodic_weights}, we can accept the whole orbit by properly weighting each sample, and we can verify that by accepting the whole orbit we preserve the target measure by considering the kernel $k(x'\cond x) = \sum_{i=0}^{T-1} \omega_i \delta(x'-f^i(x))$, where $\omega_i$ is given by \eqref{eq:periodic_weights}.
The practical value of this procedure comes from the fact that the cost of weight evaluations is equivalent to the cost of the test in Theorem \ref{th:o_test}.

The kernel behaves differently for aperiodic orbits.
Indeed, starting from the point $x_0$, the kernel always moves the probability mass further along the orbit escaping any given point on the orbit since there is no mass flowing backward (the weight of $f^{-1}(x_0)$ is zero from the beginning).
Intuitively, we can think that the weights $\omega_i^t$ from \eqref{eq:recurrence_p_t} evolve in time as a wave packet moving on the real line. 
With this intuition it becomes evident that the time average at every single point on the orbit converges to zero since the mass always escapes this point.
Therefore, if we want to accept several points from the orbit, we need to wait until the kernel ``leaves'' this set of points and then stop the procedure.
We derived a formula for this situation, which turns out to be a deterministic version of SNIS \citep{andrieu2003introduction}, and provide a simple example of this procedure in Appendix \ref{app:oSNIS}.

\subsection{Practical algorithm}

Based on the previous derivations, we propose a practical sampling scheme.
Starting from some initial state $x_0$ we would like to traverse the orbit $\orb(x_0)$ and accept all the states $\{f^i(x_0)\}_{i \in \mathbb{Z}}$ with the weights $\omega_i$ derived in Proposition \ref{th:limiting_case_escaping}.
To guarantee the unbiasedness of the procedure, we design a deterministic function with periodic orbits based on $f$.

Given a function $f$, one can easily construct a periodic function by introducing the auxiliary discrete variable $d \in \{0,\ldots,T-1\}$ (direction), which decides how we apply $f$ to the current state $x$.
That is, we design $\widehat{f}(x,d)$ on the extended space of tuples $[x,d]$ with orbits of period $T$ as follows.
\begin{align*}
    \widehat{f}(x,d) = \begin{cases}
    [f(x), (d+1) \text{ mod } T], \text{ if } d < T-1 \\
    [f^{-(T-1)}, (d+1) \text{ mod } T], \text{ if } d = T-1
    \end{cases}
\end{align*}
Indeed, starting from any pair $[x,d]$ and iteratively applying $\widehat{f}(x,d)$, we have
\begin{align}
    \widehat{f}^T(x,d) = \widehat{f}^{d+1}\bigg([f^{T-1-d}(x), T-1]\bigg) = \widehat{f}^{d}\bigg([f^{-d}(x), 0]\bigg) = [f^{d}(f^{-d}(x)), d] = [x,d].
\end{align}

\begin{wrapfigure}{r}{0.5\textwidth}
\begin{minipage}[H]{0.5\textwidth}
\vspace{-25pt}
\begin{algorithm}[H]
  \caption{Orbital MCMC (periodic)}
  \begin{algorithmic}  
    \INPUT{target density $p(x)$}
    \INPUT{density $p(v\cond x)$ and a sampler from $p(v\cond x)$} 
    \INPUT{continuous $f(x,v)$}
    \FOR{$N$ iterations}
        \STATE sample $v \sim p(v \cond x)$
        \STATE collect orbit $\{[x_j,v_j,d_j] = \widehat{f}^j(x,v,d)\}_{j=0}^{T-1}$
        \STATE $\omega_i \gets p(f^i(x,v)) \big|\partial f^i/ \partial [x,v]\big|$
        \STATE $\omega_i \gets \omega_i / \sum_j \omega_j$
        \STATE $[x,d] \gets [x_j,d_j]$ with probability $\omega_j$
        \STATE $\text{samples} \gets \text{samples} \cup \{(\omega_i,x_i)\}_{i=0}^{T-1}$
    \ENDFOR
    \OUTPUT{samples}
  \end{algorithmic}
  \label{alg:omcmc}
\end{algorithm}
\vspace{-30pt}
\end{minipage}
\end{wrapfigure}
To switch between orbits of $f(x)$ we still need another auxiliary variable to be able to cover the state space densely.
As in Algorithm \ref{alg:imcmc}, we introduce the auxiliary variable $v$ assuming that we can easily sample $v \sim p(v\cond x,d)$ and evaluate the joint density $p(x,v,d)$.
Thus, on the switching step, we sample a single point $[x_j,d_j]$ from the orbit interpreting the weights $\omega_j$ as probabilities, and then resample the auxiliary variable $v$.
For simplicity, we consider $p(d) = \text{Uniform} \{0,\ldots,T-1\}$, and $p(x,v,d) = p(x,v)p(d)$.
Gathering all of the steps we get Algorithm \ref{alg:omcmc}.

Note that the directional variable forces the chain to perform both ``forward'' and ``backward'' moves, which can be considered as a limitation of this algorithm since it may increase the autocorrelation.
To alleviate this effect we make the chain irreversible by shifting the direction by $T/2$ after each iteration.
In practice, this update performs better than uniform sampling of $d$.

\section{Diffusing orbital kernel}
\label{sec:diffusing_orbital_kernel}
We start this section by illustrating the limits of the escaping orbital kernel through a simple example.
Consider the map $f(x) = x + 1$ in $\mathbb{R}$, which has only infinite orbits.
Then the test of the escaping kernel $g(x) = \infimum_{k\in \mathbb{Z}} \{p(f^k(x))/p(x)\big|\partial f^k/ \partial x\big|\} = \infimum_{k\in \mathbb{Z}} \{p(x + k)/p(x)\} = 0$ for any target density $p(x)$. 
Hence, we can't use the map $f(x) = x + 1$ to sample even a single point from the target distribution because the density goes to zero faster than the Jacobian.
At the same time, if we just try to use the formula \eqref{eq:periodic_weights} saying that the period of the orbit is infinite, we end up with a valid kernel.
That is, $k(x'\cond x) = \sum_{i=-\infty}^{+\infty}\omega_i\delta(x'-f^i(x))$, where $\omega_i = p(f^i(x))/\sum_{j=-\infty}^{+\infty} p(f^j(x))$. 
Inserting this kernel into $Kp = p$, one can make sure that this kernel indeed keeps the target measure invariant.
This example motivates another kernel, which is able to traverse infinite trajectories:
\begin{align}
    k(x'\cond x) =& \delta(x'-f(x))g^+(x) + \delta(x'-f^{-1}(x))g^-(x)  +\delta(x-x')(1-g^+(x)-g^-(x)).
    \label{eq:diffusive_kernel}
\end{align}
Writing the condition $Kp=p$ for this kernel, we derive two possible solutions (see Appendix \ref{app:diffusing_kernel}).
One of the solutions is a linear combination of escaping orbital kernels that we have already mentioned before (Appendix \ref{app:linear_combination}).
Another solution yields a new acceptance test:
\begin{align}
    g^+(x) = p(f(x))\bigg|\frac{\partial f}{\partial x}\bigg| c(x), \;\;\;  g^-(x) = p(f^{-1}(x))\bigg|\frac{\partial f^{-1}}{\partial x}\bigg| c(x),
    \label{eq:diffusive_tests}
\end{align}
where $c(x)$ may be chosen as
\begin{align}
     c = \frac{1}{2}\infimum_{k \in \mathbb{Z}}\bigg\{\frac{1}{p(f^k(x))\big|\frac{\partial f^{k}}{\partial x}\big|}\bigg\}, \;\text{ or }\; 
     c = \infimum_{k \in \mathbb{Z}}\bigg\{\frac{1}{p(f^{k+1}(x))\big|\frac{\partial f^{k+1}}{\partial x}\big|+p(f^{k-1}(x))\big|\frac{\partial f^{k-1}}{\partial x}\big|}\bigg\}\notag
\end{align}
the second one is optimal in the sense of the minimum rejection probability, but we further proceed with the first one for simplicity.
Note that taking $f$ as an involution, the kernel \eqref{eq:diffusive_kernel} becomes equivalent to the iMCMC kernel \eqref{eq:imcmc_kernel}.
In Appendix \ref{app:reversibility_diffusing}, we prove that the diffusing orbital kernel is always reversible.

\begin{proposition}{(Reversibility)\\}
The diffusing orbital kernel \eqref{eq:diffusive_kernel} with test \eqref{eq:diffusive_tests} is reversible w.r.t. $p$: $k(x'\cond x)p(x) = k(x\cond x')p(x')$.
\end{proposition}

To provide the reader with the intuition we again assume that the map $f$ preserves some density $q$ on the orbit $\orb(x_0)$, i.e. $q(x_0) = q(f^i(x_0)) \big|\partial f^i/ \partial x\big|$.
We can then rewrite the test slightly differently:
\begin{align}
    \begin{split}
    g^+(x) = \widehat{g}(f(x)),\; g^-(x) = \widehat{g}(f^{-1}(x)),
    \text{where }\; \widehat{g}(x) = \widehat{c}\frac{p(x)}{q(x)},\;\; \widehat{c} = \frac{1}{2} \inf_k \bigg\{ \frac{q(f^k(x_0))}{p(f^k(x_0))} \bigg\}.
    \end{split}
\end{align}
Firstly, it is now apparent that the diffusing orbital kernel is tightly related to rejection sampling.
Indeed, $\widehat{g}(x)$ accepts samples with a probability proportional to the target density, and the constant $c$ makes this probability smaller than $1$: $\widehat{c} \leq q(x)/(2p(x))$.

Secondly, we can compare this test with the test for the escaping orbital kernel:
\begin{align}
    g(x) = c\cdot\frac{q(x)}{p(x)}, \;\;  c=\infimum_{k \in \mathbb{Z}}\bigg\{\frac{p(f^k(x))}{q(f^k(x))}\bigg\}.
\end{align}
We see that both tests are complementary to each other.
That is, if $p(x)/q(x)$ vanishes then $c = 0$ and we cannot use the escaping orbital kernel; however, at the same time, $\widehat{c}$ does not go to zero and we can use the diffusing orbital kernel.
The same logic applies in the opposite direction when $q(x)/p(x)$ vanishes.
In practice, having a continuous bijection $f$ one can decide between these kernels by estimating $\infimum\{p(x)/q(x)\}$ and $\infimum\{q(x)/p(x)\}$ on the orbit.
Finally, the inversion of the density ratio under the infimum allows the diffusing kernel to converge on aperiodic orbits as stated in the following proposition.
\begin{theorem}{(Convergence on a single orbit)\\}
\label{th:limiting_case_diffusing}
Consider the diffusing orbital kernel (with test \eqref{eq:diffusive_tests}, and $c > 0$), and the initial distribution $p_0(x) = \delta(x-x_0)$.
For aperiodic orbits, if the series $\sum_{j=-\infty}^{+\infty}p(f^j(x_0))\big|\frac{\partial f^{j}}{\partial x}\big|$ converges, we have
\begin{align}
    \lim_{t\to \infty} \frac{1}{t}\sum_{i=0}^{t-1}K^i p_0 = \sum_{i=-\infty}^{+\infty} \omega_i \delta(x-f^i(x_0)), \;\;
    \omega_i = \frac{p(f^i(x_0))\big|\frac{\partial f^{i}}{\partial x}\big|_{x=x_0}}{\sum_{j=-\infty}^{+\infty}p(f^j(x_0))\big|\frac{\partial f^{j}}{\partial x}\big|_{x=x_0}}.
\end{align}
For periodic orbits with period $T$, we have
\begin{align}
    \lim_{t\to \infty} \frac{1}{t}\sum_{i=0}^{t-1}K^i p_0 = \sum_{i=0}^{T-1} \omega_i \delta(x-f^i(x_0)), \;\; 
    \omega_i = \frac{p(f^i(x_0))\big|\frac{\partial f^{i}}{\partial x}\big|_{x=x_0}}{\sum_{j=0}^{T-1}p(f^j(x_0))\big|\frac{\partial f^{j}}{\partial x}\big|_{x=x_0}}.
\end{align}
\end{theorem}
\vspace{-15pt}
\begin{proof}
See Appendix \ref{app:diffusive_mean_ergodic}.
\end{proof}

\vspace{5pt}
\subsection{Practical algorithm}
\begin{wrapfigure}{r}{0.5\textwidth}
\begin{minipage}[H]{0.5\textwidth}
\vspace{-25pt}
\begin{algorithm}[H]
  \caption{Orbital MCMC (contracting)}
  \begin{algorithmic}  
    \INPUT{target density $p(x)$}
    \INPUT{density $p(v\cond x)$ and a sampler from $p(v\cond x)$} 
    \INPUT{continuous $f(x,v)$}
    \INPUT{threshold value $W$}
    \FOR{$N$ iterations}
        \STATE sample $v \sim p(v \cond x)$
        \STATE $\omega_{\text{max}} \gets p(x,v)$
        \WHILE{$\log \omega_i > \log \omega_{\text{max}} - \log W$}
        \STATE $[x_i,v_i] = f^i(x,v)$
        \STATE $\omega_{i} \gets p(f^i(x,v)) \big|\partial f^i/ \partial [x,v]\big|$ 
        \ENDWHILE
        \STATE $\omega_i \gets \omega_i / \sum_j \omega_j$
        \STATE $[x] \gets [x_i]$ with probability $\omega_i$
        \STATE $\text{samples} \gets \text{samples} \cup \{(\omega_i,x_i)\}$
    \ENDFOR
    \OUTPUT{samples}
  \end{algorithmic}
  \label{alg:omcmc_contracting}
\end{algorithm}
\vspace{-30pt}
\end{minipage}
\end{wrapfigure}

Although, the derived kernel allows us to sample using the infinite trajectories, it still puts some restrictions on the choice of the diffeomorphism $f$.
The obvious restriction is that the sum $\sum_{j=-\infty}^{+\infty}p(f^j(x_0))\big|\partial f^{j}/\partial x\big|$ has to converge.
The most straightforward way to make this series convergent is to choose $f$ as a step of an optimization algorithm, maximizing the log-probability of the target distribution: $\log p(x)$.

Indeed, when the kernel goes to minus infinity ($i \to -\infty$) the density vanishes $p(f^i(x_0)) \to 0$ since the optimizer $f$ minimizes the density going backward.
On the other hand, when $i \to +\infty$ the point converges to some local maximum $f^i(x_0) \to x^*$, hence, for proper target distributions, the density converges to some positive constant $p(f^i(x_0)) \to p(x^*)$. However, at the same time, the Jacobian vanishes $\big|\partial f^{i}/\partial x\big|\to 0$ since the optimizer is a contractive map, implying that their product also vanishes.

The convergence of weights allows us to truncate the infinite trajectory in both directions. Thus, we can iterate forward and backward from the initial state until the weights $\omega_i$ become very small. The bias introduced by the truncation can be made arbitrarily small by choosing the threshold value. In practice, we iterate while the ratio (the maximum weight)/(the current weight) is less than $10^3$ (see Algorithm \ref{alg:omcmc_contracting}).
Note that, due to its discrete nature, the optimization algorithm could be non-monotonic in $\omega_i$. This can be alleviated by choosing smaller step sizes and setting higher threshold $W$, what might negatively affect the performance.

Several hyperparameter settings related to the optimization dynamics are possible in this algorithm. To avoid expensive evaluations of the Jacobian, similar to HMC, we consider optimization schemes in the joint space of states $x$ and momenta $v$.
It is then easy to see that, for instance, the Jacobian of SGD with momentum is just a constant.
In practice, we use the Leap-Frog integrator from \citep{francca2020conformal}, which simulates Hamiltonian dynamics with friction (see Appendix \ref{app:oHMC}).

\section[Empirical study]{Empirical study\footnote{The code reproducing all experiments is available at \url{https://github.com/necludov/oMCMC}}}
\label{sec:practice}

To study the Algorithms \ref{alg:omcmc} and \ref{alg:omcmc_contracting}, we consider Hamiltonian dynamics for the deterministic function $f$.
Then the joint density $p(x,v,d)$ is the fully-factorized distribution $p(x,v,d) = p(x)\mathcal{N}(v\cond 0,\mathbbm{1}) p(d)$, where $p(d) = \text{Uniform} \{0,\ldots,T-1\}$.
For Algorithm \ref{alg:omcmc}, the deterministic transition $f$ is the Leap-Frog integrator.
For Algorithm \ref{alg:omcmc_contracting}, we choose the Leap-Frog as well, but with a friction component.
We compare both algorithms to HMC and the recycled HMC \citep{nishimura2020recycling}.
Recycled HMC simulates the whole trajectory using the Leapfrog integrator and then, for each point $f^i(x,v)$ of the trajectory, decides whether to collect the sample or not.

To tune the hyperparameters for all algorithms we use the ChEES criterion \citep{hoffman2021adaptive}. During the initial period of adaptation, this criterion optimizes the maximum trajectory length $T_{\text{max}}$ for the HMC with jitter (trajectory length at each iteration is sampled $\sim \text{Uniform}(0,T_{\text{max}})$). To set the stepsize of HMC we follow the common practice of keeping the acceptance rate around $0.65$ as suggested in \citep{beskos2013optimal}.
We set this stepsize via double averaging as proposed in \citep{hoffman2014no} and considered in ChEES-HMC.
Note that this choice of hyperparameters is designed especially for HMC and doesn't generalize to other algorithms. We leave the study of adaptation procedures for our algorithms as a future work.
For Algorithm \ref{alg:omcmc_contracting} we don't need to set the trajectory length, but we use the step size yielded at the adaptation step of ChEES-HMC. The crucial hyperparameter for this algorithm is the friction coefficient $\beta$, which we set to $\sqrt[n]{0.8}$, where $n$ is the number of dimensions of the target density, thus setting the contraction rate to $0.64$.

\begin{figure*}[t]
    \centering
    \includegraphics[width=0.24\textwidth]{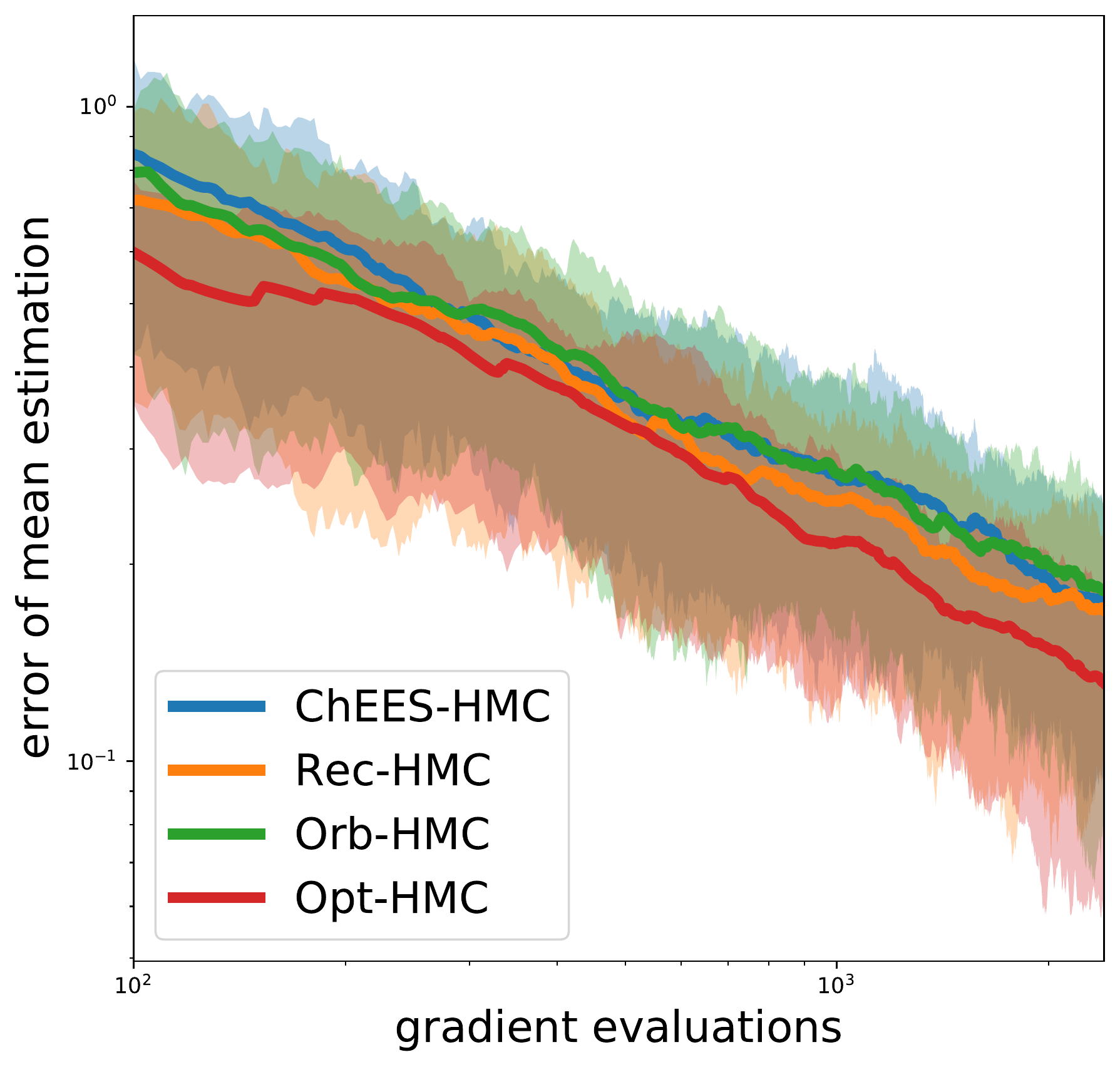}
    \includegraphics[width=0.24\textwidth]{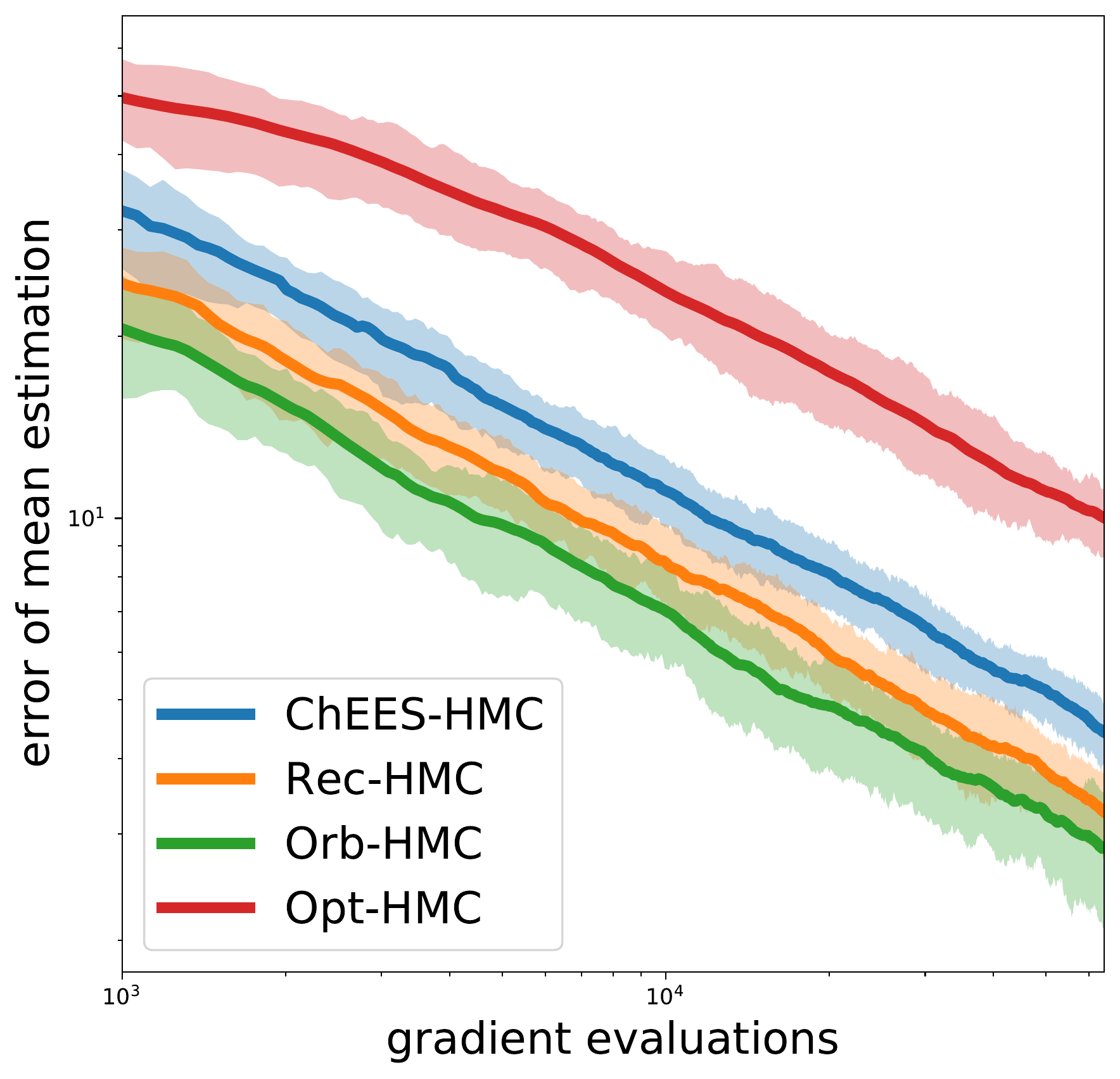}
    \includegraphics[width=0.24\textwidth]{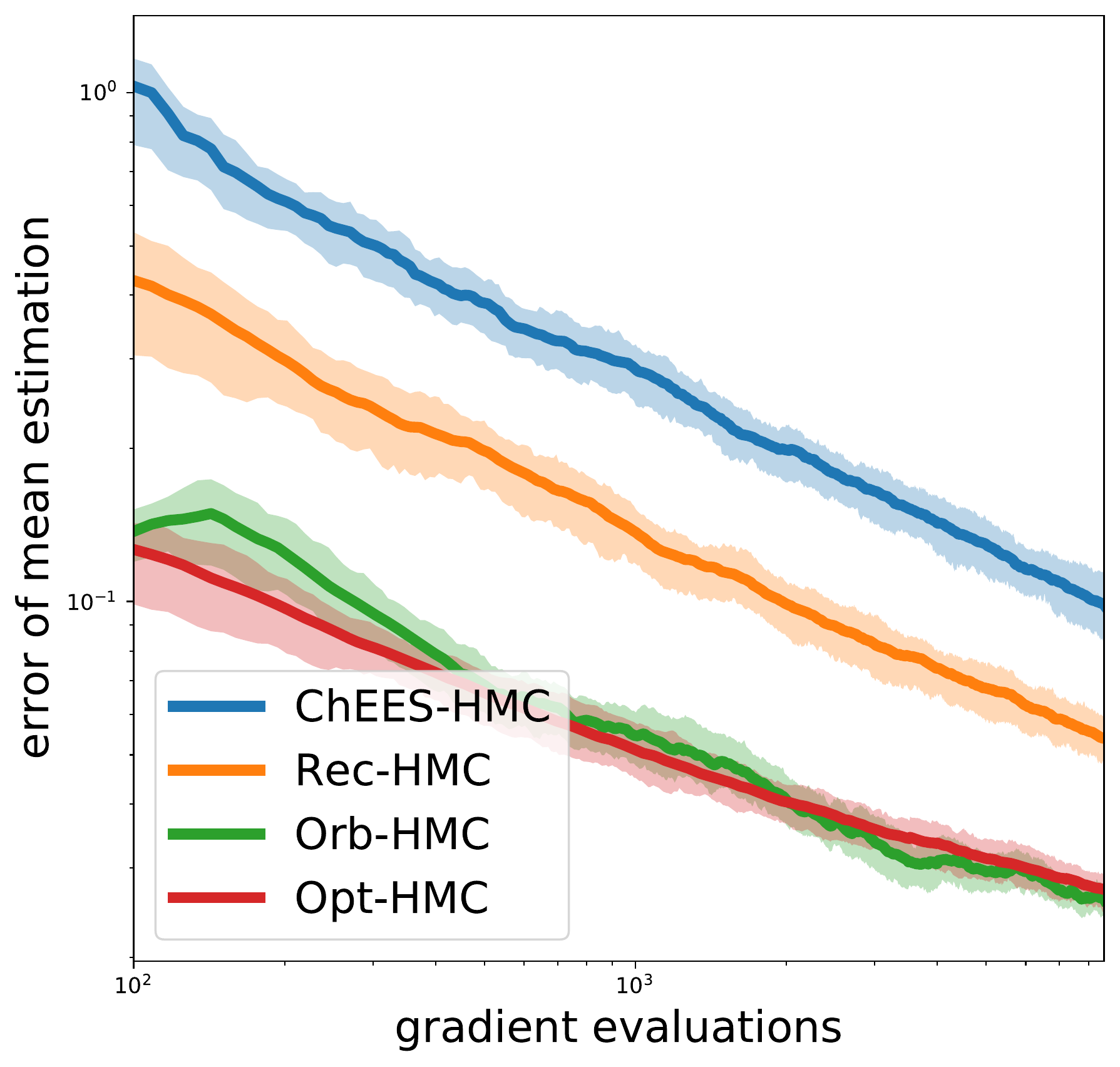}
    \includegraphics[width=0.24\textwidth]{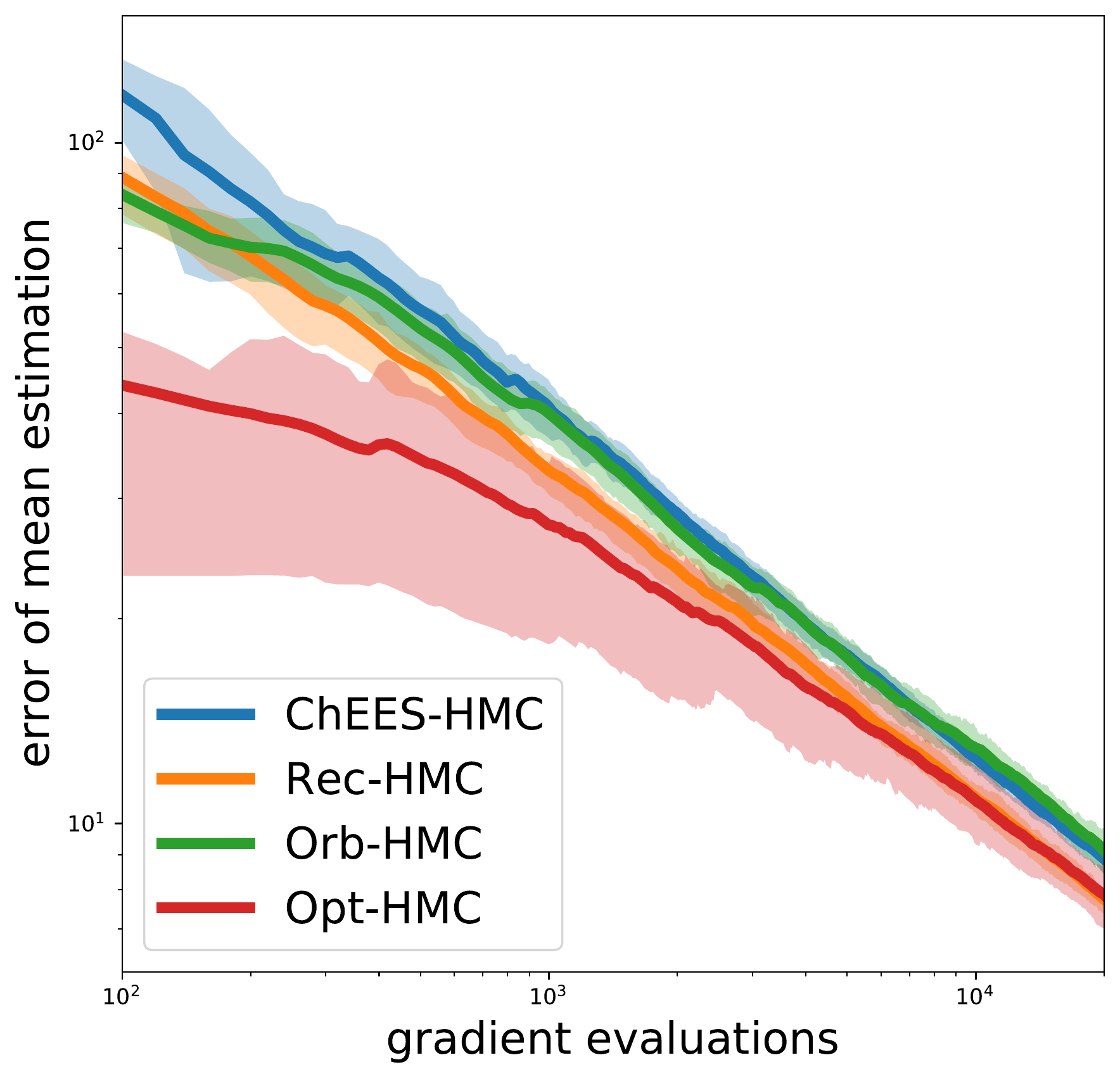}
    \caption{From left to right: the error of mean estimation on Banana, ill-conditioned Gaussian, logistic regression, Item-Response model. Every solid line depicts the mean of the absolute error averaged across $100$ independent chains. The shaded area lies between $0.25$ and $0.75$ quantiles of the error. Orb-HMC corresponds to Algorithm \ref{alg:omcmc} with Hamiltonian dynamics. Opt-HMC corresponds to Algorithm \ref{alg:omcmc_contracting} with Hamiltonian with friction.}
    \label{fig:errors}
\end{figure*}

For the comparison, we take several target distributions: Banana (2-D), ill-conditioned Gaussian (50-D), the posterior distribution of the Bayesian logistic regression (25-D), and the posterior distribution of the Item-Response model (501-D) (see description in Appendix \ref{app:dists}).
For all experiments, we use $1000$ adaptation iterations of ChEES-HMC then followed by $1000$ sampling iterations of ChEES-HMC. 
Then we fix the computational budget in terms of density and gradient evaluations and run all algorithms using approximately the same computational budget.

In Figure \ref{fig:errors}, we compare the errors in the estimation of the mean of the target distribution as a function of the number of gradient evaluations (which we take as a hardware-agnostic estimation of computation efforts).
Algorithm \ref{alg:omcmc_contracting} (Opt-HMC) provides the best estimate for the mean value for all distributions except for the ill-conditioned Gaussian (where Orbital-HMC demonstrates the fastest convergence).
We can understand the low performance of Opt-HMC for the ill-conditioned Gaussian because of its inability to traverse long distances due to the introduced friction.
This property is crucial here since the variances along all dimensions scale logarithmically from $10^{-2}$ to $10^2$ forcing the chain to keep a small step size.
Note that the ChEES criterion was specifically developed for HMC to solve such problems.
Another downside of Opt-HMC is that it relatively poorly estimates the variance of the target (we provide corresponding plots in Appendix \ref{app:oHMC}).

Given the same computational budget, all the algorithms produce different amount of samples. 
However, the evaluation of statistics of interest could be expensive. Therefore, it is important for algorithms to output a limited amount of low-correlated samples.
We validate this property by subsampling the states from the trajectory to yield the same amount of samples as HMC ($10^3$ samples).
We evaluate the Effective Sample Size (ESS) and report the ESS per gradient evaluation (including the adaptation cost into the budget of all algorithms).
The results are provided in Table \ref{tab:ess}.
Our algorithms perform relatively poorly on the Gaussian, but perform comparably or better on other distributions.
For the Item-Response model, Opt-HMC demonstrates the fastest convergence to the mean, but ChEES-HMC outperforms it in terms of ESS.
Thus, the main advantage of the proposed algorithms comes when we are able to use all of the collected samples for the estimation.

\begin{table*}[t]
\caption{Performance of the algorithms as measured by the Effective Sample Size (ESS) per gradient evaluation (higher values are better). For each algorithm we subsample $1000$ points from the trajectory preserving their order. For each of the $100$ independent chains we measure the minimum ESS across dimensions and report the median across chains as well as their standard deviations.}
\label{tab:examples}
\vspace{-1em}
\begin{center}
\resizebox{1.0\textwidth}{!}{
\begin{tabular}{lcccc}
    \toprule
    {Algorithm}&{Banana}&{Gaussian}&{Logistic Reg}&{Item-response}\\
    \midrule
    ChEES-HMC & 
    7.83e–5$\pm$6.95e–5 & 
    \textbf{1.51e–06}$\mathbf{\pm}$\textbf{6.14e–7} &
    1.52e–5$\pm$6.90e–6 & 
    \textbf{2.97e–6}$\mathbf{\pm}$\textbf{1.02e–6} \\
    Recycled-HMC & 
    7.76e–5$\pm$9.86e–5 & 
    1.43e–6$\pm$7.37e–7 & 
    2.08e–5$\pm$9.28e–6 & 
    2.63e–6$\pm$8.96e–7 \\
    Orbital-HMC & 
    7.71e–5$\pm$8.68e–5 & 
    1.08e–6$\pm$7.30e–7 & 
    3.09e–5$\pm$1.30e–5 & 
    1.73e–6$\pm$7.21e–7 \\
    Opt-HMC & 
    \textbf{2.23e–4}$\mathbf{\pm}$\textbf{2.41e–4} & 
    2.04e–7$\pm$8.90e–8 & 
    \textbf{3.66e–5}$\mathbf{\pm}$\textbf{1.82e–5} & 
    2.68e–6$\pm$1.53e–6 \\
    \bottomrule
\end{tabular}}
\end{center}
\label{tab:ess}
\vspace{-10pt}
\end{table*}

\section{Conclusion}
In this paper we have developed two new practical MCMC algorithms based on iterative deterministic maps. We believe that our orbital MCMC framework opens the door to cross-fertilization between (possibly chaotic) dynamical systems theory, optimization, and MCMC algorithm design.
In Appendix \ref{app:neural_models}, we discuss possible applications of the oMCMC framework in the context of neural models.


\bibliography{aistats2020}
\bibliographystyle{icml2021}

\newpage

\appendix

%
%





%

%


\section{Proofs}

\subsection{Orbital test}
\label{app:o_test_proof}

Consider the transition kernel
\begin{align}
    k(x'\cond x) = \delta(x'-f^m(x)) g(x) + \delta(x'-x)(1-g(x)), \;\; g(x) = \infimum_{k \in \mathbb{Z}}\bigg\{\frac{p(f^k(x))}{p(x)} \bigg| \frac{\partial f^k}{\partial x}\bigg| \bigg\}
\end{align}
Substituting this kernel into the stationary condition
\begin{align}
    \int dx \; k(x'\cond x)p(x) = p(x'),
\end{align}
we obtain
\begin{align}
    \int dx\; \delta(x'-f^m(x)) \infimum_{k \in \mathbb{Z}}\bigg\{p(f^k(x)) \bigg| \frac{\partial f^k}{\partial x}\bigg| \bigg\} + \int dx\; \delta(x'-x) \bigg(p(x)- \infimum_{k \in \mathbb{Z}}\bigg\{p(f^k(x)) \bigg| \frac{\partial f^k}{\partial x}\bigg| \bigg\}\bigg) = p(x').
\end{align}
This implies
\begin{align}
     \infimum_{k \in \mathbb{Z}}\bigg\{p(f^{k-m}(x')) \bigg| \frac{\partial f^{k}}{\partial x}\bigg|_{x=f^{-m}(x')} \bigg\} \bigg| \frac{\partial f^{-m}}{\partial x'}\bigg| &= \infimum_{k \in \mathbb{Z}}\bigg\{p(f^k(x')) \bigg| \frac{\partial f^k}{\partial x'}\bigg| \bigg\} \\
     \infimum_{k \in \mathbb{Z}}\bigg\{p(f^{k-m}(x')) \bigg| \frac{\partial f^{k-m}}{\partial x'}\bigg| \bigg\} &= \infimum_{k \in \mathbb{Z}}\bigg\{p(f^k(x')) \bigg| \frac{\partial f^k}{\partial x'}\bigg| \bigg\}.
\end{align}
The last equation holds since the shift of the index does not affect the set of lower bounds.

\subsection{Returning orbit}
\label{app:returning_orbit}

We call orbit $\orb(x_0)$ \textit{returning} if for any $x \in \orb(x_0)$ we have $\infimum_{k>0}\norm{f^k(x)-x} = 0$.
We also can think of this type of orbits as dense orbits on itself.
Note that returning orbits is a wider class than periodic orbits.
For instance, all of the orbits of $f(x) = (x + a) \;\text{mod}\; 1$ in $[0,1]$ are returning for any irrational $a$, but not periodic.

\begin{proposition*}{(Returning orbit)\\}
For continuous target density $p(x)$, consider the kernel \eqref{eq:o_kernel} from Theorem \ref{th:o_test}. 
If the orbit $\orb(x)$ is returning and $f$ preserves a measure with continuous density, then the acceptance test can be written as
\begin{align}
    g(x) = \infimum_{k \in \mathbb{Z}}\bigg\{\frac{p(f^k(x))}{p(x)} \bigg| \frac{\partial f^k}{\partial x}\bigg| \bigg\} = \infimum_{k \geq 0}\bigg\{\frac{p(f^k(x))}{p(x)} \bigg| \frac{\partial f^k}{\partial x}\bigg| \bigg\}.
\end{align}
\end{proposition*}

Here we provide the proof for the case of returning orbits ($\forall x \in \orb \infimum_{k > 0} \norm{f^k(x) - x} = 0$) that are not periodic, i.e. there is no such $k > 0$ that $f^k(x) = x$.
By the assumption the deterministic map $f$ preserves some measure with continuous density $q(x)$ on the orbit $\orb(x)$:
\begin{align}
    q(x) = \bigg| \frac{\partial f^{k}}{\partial x}\bigg| q(f^k(x)).
\end{align}
Then, we can rewrite the acceptance test from Theorem \ref{th:o_test} as
\begin{align}
    g(x) = \infimum_{k \in \mathbb{Z}}\bigg\{\frac{p(f^{k}(x))}{p(x)} \bigg| \frac{\partial f^{k}}{\partial x}\bigg| \bigg\} = \frac{q(x)}{p(x)} \infimum_{k \in \mathbb{Z}}\bigg\{\frac{p(f^{k}(x))}{q(f^k(x))}\bigg\}
\end{align}
Using the ``returning'' property of the orbit $\orb(x)$ we can prove
\begin{align}
    \frac{p(f^n(x))}{q(f^n(x))} \geq  \infimum_{k\geq0}\bigg\{\frac{p(f^k(x))}{q(f^k(x))} \bigg\} \;\; \forall n \in \mathbb{Z}.
\end{align}
Indeed, for $n \geq 0$, this holds by the definition of infimum.
For $n < 0$, we consider $x^* = f^{n}(x)$ as a starting point, and use the definition of the returning orbit: $\infimum_{k > 0} \norm{f^k(x^*) - x^*} = 0$.
We define the subsequence of the sequence $\{f^k(x^*)\}$ that converges to $x^*$ as follows.
Let $k_i$ be the first number greater than zero such that $\norm{f^{k_i}(x^*) - x^*} \leq 1/i$, then $\lim_{i\to\infty} f^{k_i}(x^*) = x^*$.
This infinite subsequence exists since $\infimum_{k > 0} \norm{f^k(x^*) - x^*} = 0$ but there is no such $k > 0$ that $f^k(x) = x$.
By continuity of $p$ and $q$,
\begin{align}
    \lim_{i\to\infty} \frac{p(f^{k_i}(x^*))}{q(f^{k_i}(x^*))} = \frac{p(x^*)}{q(x^*)} = \frac{p(f^{n}(x))}{q(f^{n}(x))}.
\end{align}
By the definition of the infimum,
\begin{align}
     \frac{p(f^{k_i}(x^*))}{q(f^{k_i}(x^*))} \geq \infimum_{k\geq0}\bigg\{\frac{p(f^k(x))}{q(f^k(x))} \bigg\} \;\; k_i > -n, \;\;\implies\;\; \frac{p(f^{n}(x))}{q(f^{n}(x))} \geq \infimum_{k\geq0}\bigg\{\frac{p(f^k(x))}{q(f^k(x))} \bigg\}.
\end{align}
Since
\begin{align}
    \frac{p(f^{n}(x))}{q(f^{n}(x))} \geq \infimum_{k\geq0}\bigg\{\frac{p(f^k(x))}{q(f^k(x))} \bigg\} \;\; \forall n \in \mathbb{Z},
\end{align}
then the infimum over the positive $k$ is a lower bound for $p(f^{n}(x))/q(f^{n}(x)) \;\; \forall n$.
Since the infimum is the maximal lower bound, we have
\begin{align}
    \infimum_{k\in \mathbb{Z}}\bigg\{\frac{p(f^k(x))}{q(f^k(x))} \bigg\} \geq \infimum_{k\geq0}\bigg\{\frac{p(f^k(x))}{q(f^k(x))} \bigg\}.
\end{align}
At the same time the infimum over all integers is less than the infimum over positive integers.
Thus, we have
\begin{align}
    \infimum_{k\in \mathbb{Z}}\bigg\{\frac{p(f^k(x))}{q(f^k(x))} \bigg\} = \infimum_{k\geq0}\bigg\{\frac{p(f^k(x))}{q(f^k(x))} \bigg\}.
\end{align}

\subsection{Reversibility}

\subsubsection{Single kernel}
\label{app:o_kernel_reversiblity}

Consider the kernel
\begin{align}
    k_m(x'\cond x) = \delta(x'-f^m(x))g(x) + \delta(x'-x)(1-g(x)),
\end{align}
and assume that it satisfies the fixed point equation ($\int dx\; k_m(x'\cond x)p(x) = p(x')$) what reduces to the system of equations at each point $x'$:
\begin{align}
    g(x')p(x') = g(f^{k}(x'))p(f^{k}(x'))\bigg|\frac{\partial f^{k}}{\partial x}\bigg|_{x=x'} \;\; \forall k \in \mathbb{Z}.
\end{align}
Since it satisfies the fixed point equation, we can define the reverse kernel
\begin{align}
    r_m(x'\cond x) = k_m(x\cond x')\frac{p(x')}{p(x)} = \delta(x-f^m(x')) \frac{p(x')}{p(x)}g(x') + \delta(x-x') \bigg(\frac{p(x')}{p(x)} - \frac{p(x')}{p(x)}g(x')\bigg).
\end{align}
Integrating around the point $x'=x$, we obtain (assuming that $f^{-m}(x)\neq x$)
\begin{align}
    \int_{A(x)} dx'\;r_m(x'\cond x) = (1-g(x)) = \int_{A(x)}dx'\;k_m(x'\cond x).
\end{align}
Integrating around the point $x'=f^{-m}(x)$, we obtain
\begin{align}
    \int_{A(f^{-m}(x))}dx'\;r_m(x'\cond x) &= \int_{A(f^{-m}(x))}dx'\; \delta(x-f^m(x')) g(x')\frac{p(x')}{p(x)} = (x' = f^{-m}(y))\\
    &= \int_{f^{m}(A(f^{-m}(x)))} dy\;\delta(x-y) g(f^{-m}(y))\frac{p(f^{-m}(y))}{p(x)} \bigg| \frac{\partial f^{-m}}{\partial y} \bigg| \\
    &= g(f^{-m}(x))\frac{p(f^{-m}(x))}{p(x)} \bigg| \frac{\partial f^{-m}}{\partial x} \bigg| = g(x)
\end{align}
And for the forward kernel
\begin{align}
    \int_{A(f^{-m}(x))}dx'\;k_m(x'\cond x) = \int_{A(f^{-m}(x))}dx'\delta(x'-f^m(x))g(x)
\end{align}
Since $m$ is a fixed number, we can choose the radius of $A(f^{-m}(x))$ small enough that $f^m(x)$ does not lie in $A(f^{-m}(x))$.
However, if $f^m(x)=f^{-m}(x)$, we have
\begin{align}
    \int_{A(f^{-m}(x))}dx'\;k_m(x'\cond x) = g(x) = \int_{A(f^{-m}(x))}dx'\;r_m(x'\cond x).
\end{align}
Hence, the kernel $k_m(x'\cond x)$ is reversible if and only if $f^{-m}(x') = f^{m}(x')$.

\subsubsection{Linear combination}
\label{app:whole_orbit_reversibility}

Consider the kernel
\begin{align}
    k_m(x'\cond x) = \delta(x'-f^m(x))g_m(x) + \delta(x'-x)(1-g_m(x)),
\end{align}
that preserves target measure $p(x)$ for any integer $m$.
Then the linear combination
\begin{align}
    k(x'\cond x) = \sum_{m \in \mathbb{Z}} w_m k_m(x'\cond x), \;\; \sum_{m \in \mathbb{Z}} w_m = 1, \;\; w_m \geq 0
\end{align}
also preserves the target measure.
Hence, we can define the reverse kernel as
\begin{align}
    r(x'\cond x) = k(x\cond x') \frac{p(x')}{p(x)} = \sum_{m\in \mathbb{Z}} w_m r_m(x'\cond x),
\end{align}
where $r_m$ is the reverse kernel of $k_m$ as we derived in Appendix \ref{app:o_kernel_reversiblity}:
\begin{align}
    r_m(x'\cond x) = k_m(x\cond x')\frac{p(x')}{p(x)} = \delta(x-f^m(x')) \frac{p(x')}{p(x)}g_m(x') + \delta(x-x') \bigg(\frac{p(x')}{p(x)} - \frac{p(x')}{p(x)}g_m(x')\bigg).
\end{align}
Further, we assume that all the points of the orbits are distinct, i.e. orbit is neither periodic nor stationary.
Integrating the kernel $r(x'\cond x)$ around the point $x'=x$, we get
\begin{align}
    \int_{A(x)} dx'\;r(x'\cond x) = \sum_{m\in \mathbb{Z}} w_m \int_{A(x)} dx'\;r_m(x'\cond x) = \sum_{m\in \mathbb{Z}} w_m(1-g_m(x)) = \int_{A(x)}dx'\;k(x'\cond x).
\end{align}
Integrating around the point $x'=f^{-m}(x)$, we obtain
\begin{align}
    \int_{A(f^{-m}(x))}dx'\;r(x'\cond x) &= \int_{A(f^{-m}(x))}dx'\; w_m\delta(x-f^m(x')) g_m(x')\frac{p(x')}{p(x)} = (x' = f^{-m}(y))\\
    &= \int_{f^{m}(A(f^{-m}(x)))} dy\;w_m\delta(x-y) g_m(f^{-m}(y))\frac{p(f^{-m}(y))}{p(x)} \bigg| \frac{\partial f^{-m}}{\partial y} \bigg| \\
    &= w_mg_m(f^{-m}(x))\frac{p(f^{-m}(x))}{p(x)} \bigg| \frac{\partial f^{-m}}{\partial x} \bigg| = w_mg_m(x)
\end{align}
And for the forward kernel
\begin{align}
    \int_{A(f^{-m}(x))}dx'\;k(x'\cond x) &= \int_{A(f^{-m}(x))}dx'\;w_{-m}k_{-m}(x'\cond x) =\\  &= \int_{A(f^{-m}(x))}dx'w_{-m}\delta(x'-f^{-m}(x))g_{-m}(x) = w_{-m}g_{-m}(x)
\end{align}
In all these derivations we integrated the kernels around a single point assuming that the area $A(x)$ around the point $x$ can be chosen so small that includes $x$ and does not include other points from the orbit.
However, this is not the case for returning orbits.
By definition, the returning orbit has a subsequence on the orbit that converges to the starting point.
Nevertheless, we can choose such small area $A(x)$ that the point weights of the subsequence inside the area $A(x)$ go to zero since all the weights are supposed to be positive and sum up to $1$.

Thus, we see that for aperiodic orbits of $f$ the linear combination of kernels $k(x'\cond x) = \sum_{m \in \mathbb{Z}} w_m k_m(x'\cond x)$ is reversible if and only if $w_{-m}g_{-m}(x) = w_{m}g_{m}(x), \;\; \forall m \in \mathbb{Z}$.
The same criterion can be formulated for the periodic orbits as well if we assume that the only positive weights are $w_1,\ldots, w_{T-1}$, where $T$ is the period of the orbit.
Then the linear combination is reversible if and only if $w_{T-m}g_{-m}(x) = w_{m}g_{m}(x), \;\; \forall m \in [1,T-1]$.

\subsection{Mean ergodic theorem for the escaping orbital kernel}
\label{app:limiting_case}

We analyse the kernel
\begin{align}
    k(x'\cond x) = \delta(x'-f(x))g(x) + \delta(x'-x)(1-g(x))
\end{align}
that preserves target measure $p(x)$, and $g(x) > 0$.
Applying this kernel to the delta function $\delta(x-x_0)$ we obtain
\begin{align}
    \int dx\; k(x'\cond x) \delta(x-x_0) = (1-g(x_0))\delta(x'-x_0) + g(x_0)\delta(x'-f(x_0)).
\end{align}
Thus, the kernel leaves a part of mass $(1-g(x_0))$ at the initial point and propagates the other part $g(x_0)$ further in orbit.
Let the initial distribution be the delta-function $p_0(x) = \delta(x-x_0)$.
Iteratively applying the kernel $k(x'\cond x)$, at time $t+1$, we obtain a chain of weighted delta-functions along the orbit $\orb(x_0)$ of $f$:
\begin{align}
    \int dx\; k(x'\cond x) p_t(x) = p_{t+1}(x') = \sum_{i=0}^{t+1} \omega_i^{t+1}\delta(x'-f^i(x_0)).
\end{align}
Denoting the points on the orbit as $x_i = f^i(x_0)$, we can write the recurrence relation for the weights of the delta-function at $x_i$ as
\begin{align}
    \omega_{i}^{t+1} = (1-g(x_{i}))\omega_i^{t} + g(x_{i-1})\omega_{i-1}^{t}, \;\; t > 0, \; i > 0.
    \label{eq:omega_rec}
\end{align}
For periodic orbits with period $T$, we consider the lower index of $\omega_{i}^{t+1}$ by modulo $T$.
Further, denoting the sum of the weights over time as
\begin{align}
    S_{i}^{t} = \sum_{t'=0}^t \omega_i^{t'},
\end{align}
we obtain the following recurrence relation by summation of \eqref{eq:omega_rec}.
\begin{align}
    S_{i}^{t+1} - \omega_i^0 = (1-g(x_{i}))S_i^{t} + g(x_{i-1})S_{i-1}^{t}
    \label{eq:sum_rec}
\end{align}
Then we denote the solution of this relation as $S^t_i = \alpha^t_i \widehat{S}_i^t$, where $\widehat{S}_i^t$ is the solution of the corresponding homogeneous relation, and $\alpha^t_i$ is the constant variation.
\begin{align}
    \alpha_{i}^{t+1}\widehat{S}_{i}^{t+1} &= (1-g(x_{i}))\alpha_{i}^{t}\widehat{S}_{i}^{t} + g(x_{i-1})S_{i-1}^{t} + \omega_i^0 \\
    \alpha_{i}^{t+1}(1-g(x_{i}))\widehat{S}_{i}^{t} &= (1-g(x_{i}))\alpha_{i}^{t}\widehat{S}_{i}^{t} + g(x_{i-1})S_{i-1}^{t} + \omega_i^0\\
    \alpha_{i}^{t+1} &= \alpha_{i}^{t} + \frac{g(x_{i-1})S_{i-1}^{t} + \omega_i^0}{(1-g(x_{i}))\widehat{S}_{i}^{t}} \\
    \alpha_{i}^{t} &= \alpha_i^0 + \sum_{t'=0}^{t-1}\frac{g(x_{i-1})S_{i-1}^{t'} + \omega_i^0}{(1-g(x_{i}))\widehat{S}_{i}^{t'}}, \;\; t > 0
\end{align}
Together with the solution $\widehat{S}_i^t = (1-g(x_i))^t\widehat{S}_i^0$ of the homogeneous relation, we have
\begin{align}
    S_{i}^{t} & = \bigg(\alpha_i^0 + \sum_{t'=0}^{t-1}\frac{g(x_{i-1})S_{i-1}^{t'} + \omega_i^0}{(1-g(x_i))^{t'+1}S_i^0}\bigg)(1-g(x_i))^{t}\widehat{S}_i^0\\
    S_{i}^{t} & = (1-g(x_i))^{t}S_i^0 + \sum_{t'=0}^{t-1}\frac{g(x_{i-1})S_{i-1}^{t'} + \omega_i^0}{(1-g(x_i))^{t'-(t-1)}}
\end{align}
The first term of the solution goes to zero.
Hence, for aperiodic orbits we have
\begin{align}
    \lim_{t \to\infty} S_0^t = \frac{1}{g(x_0)}
\end{align}
since $S_{-1}^t = 0$, and $\omega_0^0 = 1$. For $i > 0$, $\omega_i^0 = 0$, and the last sum we can split as
\begin{align}
    \lim_{t \to\infty} S_{i}^{t} = \lim_{t \to\infty} \bigg[(1-g(x_i))^{t-1} \sum_{t'=0}^{t_1}\frac{g(x_{i-1})S_{i-1}^{t'}}{(1-g(x_i))^{t'}} + \sum_{t'=t_1}^{t-1}\frac{g(x_{i-1})S_{i-1}^{t'}}{(1-g(x_i))^{t'-(t-1)}} \bigg].
\end{align}
Taking $t_1$ large enough, we obtain
\begin{align}
    \lim_{t \to\infty} S_{1}^{t} = \lim_{t \to\infty} \sum_{t'=t_1}^{t-1}\frac{g(x_0)/g(x_0)}{(1-g(x_1))^{t'-(t-1)}} = \lim_{t \to\infty} \sum_{t'=0}^{t-1-t_1}(1-g(x_1))^{t'} = \frac{1}{g(x_1)}.
\end{align}
Applying this recursively, we obtain
\begin{align}
    \lim_{t \to\infty} S_{i}^{t} = \lim_{t \to\infty} \sum_{t'=0}^t \omega_i^{t'} = \frac{1}{g(x_i)}, \;\; i \geq 0.
\end{align}
Note that the limiting behaviour of the chain on the aperiodic orbit depends on the initial conditions $\omega_0^0,\ldots,\omega_i^0,\ldots$, and each individual weight goes to zero with time: $\lim_{t\to\infty}\omega_i^t = 0$.
Thus, the total mass drifts away from the initial point down the orbit.

The chain behaves differently in the case of the periodic orbit.
Since the number of points on the orbit is finite, the total mass $S_i^t$ goes to infinity with number of iterations $t$ (if $g(x_i) < 1$).
However, the average mass over iterations remains constant
\begin{align}
    \frac{1}{t}\sum_{i=0}^{T-1}S_i^t = 1, \;\; t \geq 0,
    \label{eq:average_sum}
\end{align}
where $T$ is the period of orbit $\orb(x_0)$.
Thus, we provide the  
\begin{align}
    \lim_{t \to\infty}\frac{1}{t}S_{i}^{t} & = \lim_{t \to\infty} \frac{1}{t}\sum_{t'=0}^{t-1}\frac{g(x_{i-1})S_{i-1}^{t'} + \omega_i^0}{(1-g(x_i))^{t'-(t-1)}} = \\
    & = \lim_{t \to\infty}\bigg[g(x_{i-1})\sum_{t'=0}^{t-1}\frac{S_{i-1}^{t'}}{t}(1-g(x_i))^{(t-1)-t'} + \frac{\omega_i^0}{t}\sum_{t'=0}^{t-1}(1-g(x_i))^{t'}\bigg]
\end{align}
The second series goes to zero when $t\to\infty$.
The first series converges since $S_{i-1}^{t'}/t \leq 1, \forall t'$.
Denoting the limit of $S_{i}^{t}/t$ as $S_i$ and putting it into the original relation \eqref{eq:sum_rec}, we get
\begin{align}
    S_i = (1-g(x_i))S_i + g(x_{i-1})S_{i-1} \implies S_i = \frac{g(x_{i-1})}{g(x_i)}S_{i-1}.
\end{align}
Finally, taking the limit in \eqref{eq:average_sum}, we have
\begin{align}
    \sum_{i=0}^{T-1}S_i = 1 \implies \sum_{i=0}^{T-1}S_0 \frac{g(x_{0})}{g(x_i)} = 1 \implies S_0 = \frac{\frac{1}{g(x_0)}}{\sum_{i=0}^{T-1}\frac{1}{g(x_i)}} \implies S_j = \frac{\frac{1}{g(x_j)}}{\sum_{i=0}^{T-1}\frac{1}{g(x_i)}}.
\end{align}
By the assumption, the kernel $k(x'\cond x)$ preserves the target distribution $p(x)$.
Hence, for the test $g(x)$, we have
\begin{align}
    \frac{g(x_j)}{g(x_i)} = \frac{p(x_i)\big|\frac{\partial f^i}{\partial x}\big|_{x=x_0}}{p(x_j)\big|\frac{\partial f^j}{\partial x}\big|_{x=x_0}}.
\end{align}
Thus, for periodic orbits we have
\begin{align}
    \lim_{t\to \infty} \frac{1}{t}S_{i}^{t} = \lim_{t\to \infty} \frac{1}{t} \sum_{t'=0}^t \omega_i^{t'} = \frac{p(f^i(x_0))\big|\frac{\partial f^{i}}{\partial x}\big|_{x=x_0}}{\sum_{j=1}^Tp(f^j(x_0))\big|\frac{\partial f^{j}}{\partial x}\big|_{x=x_0}}.
\end{align}
Note that the limit does not depend on the initial values of the weights $\omega_0^0,\ldots,\omega_{T-1}^0$.

\subsection{Average escape time}
\label{app:escape_time}
Consider the kernel
\begin{align}
    k(x'\cond x) = \delta(x'-f(x))g(x) + \delta(x'-x)(1-g(x))
\end{align}
that preserves target measure $p(x)$, and $g(x) > 0$.
Here we treat it stochastically assuming that we accept $f(x)$ with the acceptance probability $g(x)$ and stay at the same point $x$ with probability $(1-g(x))$.
Assuming that we start at some point $x_0$, we define the escape time $t_n$ as the number of iterations required to leave $x_{n-1}$, or equivalently the time of the first acceptance of $x_n = f^n(x_0)$.
Then the expectation of $t_n$ is
\begin{align}
    \mean t_n = & \sum_{t_0,\ldots,t_{n-1}}^\infty \bigg(\sum_{i=0}^{n-1}(t_i+1)\bigg) \prod_{i=0}^{n-1} g(x_{i})(1-g(x_{i}))^{t_i} = \\
    = &\sum_{t_1,\ldots,t_{n-1}}^\infty \bigg(\sum_{i=1}^{n-1}(t_i+1)\bigg) \prod_{i=0}^{n-1} g(x_{i})(1-g(x_{i}))^{t_i} + \\ 
    &+ \sum_{t_1,\ldots,t_{n-1}}^\infty \sum_{t_0=0}^\infty (t_0 + 1)\prod_{i=0}^{n-1} g(x_{i})(1-g(x_{i}))^{t_i}
\end{align}
The sum over $t_0$ in the last term is
\begin{align}
    \sum_{t_0=0}^\infty (t_0 + 1) g(x_{0})(1-g(x_{0}))^{t_0} = 1 + \underbrace{\sum_{t_0=1}^\infty t_0 g(x_{0})(1-g(x_{0}))^{t_0}}_{S},
\end{align}
where the last sum can be calculated via the index shifting trick as follows.
\begin{align}
    S & = \sum_{t_0=1}^\infty t_0 g(x_{0})(1-g(x_{0}))^{t_0} = g(x_{0})(1-g(x_{0})) + \sum_{t_0=2}^\infty t_0 g(x_{0})(1-g(x_{0}))^{t_0} = \\
    & = g(x_{0})(1-g(x_{0})) +  \sum_{t_0=1}^\infty (t_0 + 1) g(x_{0})(1-g(x_{0}))^{t_0+1} = \\ 
    & = g(x_{0})(1-g(x_{0})) + (1-g(x_{0}))\underbrace{\sum_{t_0=1}^\infty t_0 g(x_{0})(1-g(x_{0}))^{t_0}}_{S} + (1-g(x_0))^2
\end{align}
Hence,
\begin{align}
    S = (1 - g(x_0)) + \frac{1}{g(x_0)} - 2 + g(x_0) = \frac{1}{g(x_0)} - 1.
\end{align}
Substituting all the derivations back into the formula for expectation, we have
\begin{align}
    \mean t_n = &\sum_{t_1,\ldots,t_{n-1}}^\infty \bigg(\sum_{i=1}^{n-1}(t_i+1)\bigg) \prod_{i=0}^{n-1} g(x_{i})(1-g(x_{i}))^{t_i} + \\ 
    &+ \sum_{t_1,\ldots,t_{n-1}}^\infty \sum_{t_0=0}^\infty (t_0 + 1)\prod_{i=0}^{n-1} g(x_{i})(1-g(x_{i}))^{t_i} = \\
    = & \sum_{t_1,\ldots,t_{n-1}}^\infty \bigg(\sum_{i=1}^{n-1}(t_i+1)\bigg) \prod_{i=0}^{n-1} g(x_{i})(1-g(x_{i}))^{t_i} + \frac{1}{g(x_0)}.
\end{align}
Applying this reasoning $(n-1)$ more times, we obtain
\begin{align}
    \mean t_n = \sum_{i=0}^{n-1} \frac{1}{g(x_i)}.
\end{align}

\subsection{Diffusing orbital kernel}
\label{app:diffusing_kernel}

We start with the kernel
\begin{align}
    k(x'\cond x) = \delta(x'-f(x))g^+(x) + \delta(x'-f^{-1}(x))g^-(x) + \delta(x-x')(1-g^+(x)-g^-(x)).
\end{align}
Putting this kernel into $Kp = p$, we get
\begin{align}
\begin{split}
    \int dx\;\delta(x'-f(x))g^+(x) p(x) + \int dx\;\delta(x'-f^{-1}(x))g^-(x)p(x) + \\
    + p(x')(1-g^+(x')-g^-(x')) = p(x'),
\end{split}
\end{align}
which yields the following condition for $Kp = p$.
\begin{align}
    g^+(f^{-1}(x'))p(f^{-1}(x')) \bigg|\frac{\partial f^{-1}}{\partial x'}\bigg|  + g^-(f(x'))p(f(x')) \bigg|\frac{\partial f}{\partial x'}\bigg| = g^+(x')p(x')+g^-(x')p(x')
\end{align}
One of the solutions here is the linear combination of two escaping orbital kernels, which can be obtained by separately matching terms with $g^+$ and then the terms with $g^-$.
Another solution, which yields the diffusing orbital kernel, is obtained by matching 
\begin{align}
    g^+(f^{-1}(x'))p(f^{-1}(x')) \bigg|\frac{\partial f^{-1}}{\partial x'}\bigg| = g^-(x')p(x') \;\text{ and }\; g^-(f(x'))p(f(x')) \bigg|\frac{\partial f}{\partial x'}\bigg| = g^+(x')p(x').
\end{align}
Considering $x'$ as some point from the orbit $\orb(x_0)$, i.e. $x'=f^k(x_0)$, both equations yield the same system
\begin{align}
    \frac{g^+(f^{k-1}(x_0))}{g^-(f^k(x_0))} = \frac{p(f^k(x_0))}{p(f^{k-1}(x_0))}\frac{\big|\partial f^{k}/\partial x\big|_{x=x_0}}{\big|\partial f^{k-1}/\partial x\big|_{x=x_0}}, \;\; \forall k \in \mathbb{Z}.
\end{align}
Now we use the same trick as for the escaping orbital kernel assuming that $f$ preserves some measure with the density $q(x)$, i.e. $q(x_0) = q(f^k(x_0)) \big|\partial f^k/\partial x \big|_{x=x_0}$.
Expressing the Jacobians in terms of the density $q$, we obtain
\begin{align}
    \frac{g^+(f^{k-1}(x_0))}{g^-(f^k(x_0))} = \frac{p(f^k(x_0))}{p(f^{k-1}(x_0))}\frac{q(f^{k-1}(x_0))}{q(f^{k}(x_0))}, \;\; \forall k \in \mathbb{Z}.
\end{align}
Once again, there are two options to design the test function here.
One of the options is to say let $g^+(x) = g^-(x) \propto q(x)/p(x)$, but saying so we will end up with the mixture of escaping orbital kernels.
Another option is to consider the tests as follows
\begin{align}
    g^+(x) = \frac{p(f(x))}{q(f(x))} c \;\;\text{ and }\;\; g^-(x) = \frac{p(f^{-1}(x))}{q(f^{-1}(x))} c,
    \label{eq:diffusing_tests_with_q}
\end{align}
where the constant $c$ comes from the fact that the rejection probability is non-negative on the whole orbit $\orb(x_0)$, i.e. $(1-g^+(x)-g^-(x)) \geq 0$, and we have
\begin{align}
    c\bigg(\frac{p(f^{k+1}(x_0))}{q(f^{k+1}(x_0))} + \frac{p(f^{k-1}(x_0))}{q(f^{k-1}(x_0))}\bigg) \leq 1.
\end{align}
Thus, we end up with the following value of $c$
\begin{align}
    c = \infimum_{k\in \mathbb{Z}}\bigg\{\frac{1}{p(f^{k+1}(x_0))/q(f^{k+1}(x_0)) + p(f^{k-1}(x_0))/q(f^{k-1}(x_0))}\bigg\},
\end{align}
or with another lower bound, which may be not optimal, but useful in practice:
\begin{align}
    c' = \frac{1}{2}\infimum_{k\in \mathbb{Z}}\bigg\{\frac{q(f^{k}(x_0))}{p(f^{k}(x_0))}\bigg\} \leq c.
\end{align}
Putting these constants back into the tests \eqref{eq:diffusing_tests_with_q}, we get
\begin{align}
    g^+(x) = p(f(x))\bigg|\frac{\partial f}{\partial x}\bigg| c(x), \;\; g^-(x) = p(f^{-1}(x))\bigg|\frac{\partial f^{-1}}{\partial x}\bigg| c(x),
\end{align}
where $c(x)$ may be chosen as
\begin{align}
     c'(x) = \frac{1}{2}\infimum_{k \in \mathbb{Z}}\bigg\{\frac{1}{p(f^k(x))\big|\frac{\partial f^{k}}{\partial x}\big|}\bigg\}, \;\text{ or }\; c(x) = \infimum_{k \in \mathbb{Z}}\bigg\{\frac{1}{p(f^{k+1}(x))\big|\frac{\partial f^{k+1}}{\partial x}\big|+p(f^{k-1}(x))\big|\frac{\partial f^{k-1}}{\partial x}\big|}\bigg\}.
\end{align}

\subsection{Reversibility of the diffusing orbital kernel}
\label{app:reversibility_diffusing}

Consider the diffusing orbital kernel
\begin{align}
    k(x'\cond x) = \delta(x'-f(x))g^+(x) + \delta(x'-f^{-1}(x))g^-(x) + \delta(x'-x)(1-g^+(x)-g^-(x)),
\end{align}
that preserves the target measure ($\int dx\; k_m(x'\cond x)p(x) = p(x')$) what reduces to the system of equations:
\begin{align}
    \frac{g^+(f^{k-1}(x_0))}{g^-(f^k(x_0))} = \frac{p(f^k(x_0))}{p(f^{k-1}(x_0))}\frac{\big|\partial f^{k}/\partial x\big|_{x=x_0}}{\big|\partial f^{k-1}/\partial x\big|_{x=x_0}}, \;\; \forall k \in \mathbb{Z}.
\end{align}
Since it satisfies the fixed point equation, we can define the reverse kernel
\begin{align}
    r(x'\cond x) = k(x\cond x')\frac{p(x')}{p(x)} = \delta(x-f(x')) \frac{p(x')}{p(x)}g^+(x') + \delta(x-f^{-1}(x')) \frac{p(x')}{p(x)}g^-(x') + \\ 
    + \delta(x-x') \frac{p(x')}{p(x)}(1 - g^+(x') - g^-(x')).
\end{align}
Integrating around the point $x'=x$, we obtain
\begin{align}
    \int_{A(x)} dx'\;r(x'\cond x) = (1 - g^+(x) - g^-(x)) = \int_{A(x)}dx'\;k(x'\cond x).
\end{align}
Integrating around the point $x'=f^{-1}(x)$, we obtain
\begin{align}
    \int_{A(f^{-1}(x))}dx'\;r(x'\cond x) &= \int_{A(f^{-1}(x))}dx'\; \delta(x-f(x')) g^+(x')\frac{p(x')}{p(x)} = (x' = f^{-1}(y))\\
    &= \int_{f(A(f^{-1}(x)))} dy\;\delta(x-y) g^+(f^{-1}(y))\frac{p(f^{-1}(y))}{p(x)} \bigg| \frac{\partial f^{-1}}{\partial y} \bigg| \\
    &= g^+(f^{-1}(x))\frac{p(f^{-1}(x))}{p(x)} \bigg| \frac{\partial f^{-1}}{\partial x} \bigg| = g^-(x)
\end{align}
And for the forward kernel, the integral around $x'=f^{-1}(x)$ yields
\begin{align}
    \int_{A(f^{-1}(x))}dx'\;k(x'\cond x) = \int_{A(f^{-1}(x))}dx'\delta(x'-f^{-1}(x))g^-(x) = g^-(x).
\end{align}
Thus, we have
\begin{align}
    \int_{A(f^{-1}(x))}dx'\;r(x'\cond x) = \int_{A(f^{-1}(x))}dx'\;k(x'\cond x)
\end{align}
The same reasoning applies for the integration around $x' = f(x)$, hence, we conclude that the diffusing orbital kernel is reversible since $r(x'\cond x) = k(x'\cond x)$.

\subsection{Mean ergodic theorem for the diffusing orbital kernel}
\label{app:diffusive_mean_ergodic}

We analyse the diffusive orbital kernel
\begin{align}
    k(x'\cond x) = \delta(x'-f(x))g^+(x) + \delta(x'-f^{-1}(x))g^-(x) + \delta(x-x')(1-g^+(x)-g^-(x))
\end{align}
that preserves target measure $p(x)$, thus having (from Appendix \ref{app:diffusing_kernel})
\begin{align}
    g^+(x) = \frac{p(f(x))}{q(f(x))} c \;\;\text{ and }\;\; g^-(x) = \frac{p(f^{-1}(x))}{q(f^{-1}(x))} c.
\end{align}
We assume that the map $f$ preserves some density $q$ only to simplify the derivations. One can as well consider the tests derived in Appendix \ref{app:diffusing_kernel}.

Applying this kernel to the delta function $\delta(x-x_0)$ we obtain
\begin{align}
\begin{split}
    \int dx\; k(x'\cond x) \delta(x-x_0) = & g^+(x_0)\delta(x'-f(x_0)) + g^-(x_0)\delta(x'-f(x_0)) + \\
    & + (1-g^+(x_0)-g^-(x_0))\delta(x'-x_0).
\end{split}
\end{align}
That is, the kernel leaves a part of mass $(1-g^+(x_0)-g^-(x_0))$ at the initial point and moves the other part to the adjacent points $f(x_0)$ and $f^{-1}(x_0)$.
Thus, starting from the delta-function $p_0(x) = \delta(x-x_0)$, and 
iteratively applying the kernel $k(x'\cond x)$, at time $t+1$, we obtain a chain of weighted delta-functions along the orbit $\orb(x_0)$ of $f$:
\begin{align}
    \int dx\; k(x'\cond x) p_t(x) = p_{t+1}(x') = \sum_{i=-(t+1)}^{t+1} \omega_i^{t+1}\delta(x'-f^i(x_0)).
\end{align}
Denoting the points on the orbit as $x_i = f^i(x_0)$, we can write the recurrence relation for the weights of the delta-function at $x_i$ as
\begin{align}
    \omega_{i}^{t+1} = (1-g^+(x_{i})-g^-(x_{i}))\omega_i^{t} + g^+(x_{i-1})\omega_{i-1}^{t} + g^-(x_{i+1})\omega_{i+1}^t, \;\; t \geq 0.
    \label{eq:diffusive_recurrence}
\end{align}
For periodic orbits with period $T$, we consider the lower index of $\omega_{i}^{t+1}$ by modulo $T$.
Further, denoting the sum of the weights over time as
\begin{align}
    S_{i}^{t} = \sum_{t'=0}^{t-1} \omega_i^{t'},
\end{align}
we obtain the following recurrence relation by summation of \eqref{eq:diffusive_recurrence}.
\begin{align}
    S_i^{t+1} = \omega_i^0 + (1-g^+(x_{i})-g^-(x_{i}))S_i^t +  g^+(x_{i-1})S_{i-1}^{t} + g^-(x_{i+1})S_{i+1}^t
    \label{eq:recurrence_for_sum_diffusing}
\end{align}
Decomposing the solution as $S_i^t = \alpha_i^t \widehat{S}_i^t$, where $\widehat{S}_i^t$ is the solution of the corresponding homogeneous relation, and $\alpha^t_i$ is the constant variation, we get
\begin{align}
    S_i^t =  (1-g^+(x_{i})-g^-(x_{i}))^tS_i^0 + \sum_{t'=0}^{t-1}\frac{g^+(x_{i-1})S_{i-1}^{t'} + g^-(x_{i+1})S_{i+1}^{t'} + \omega_i^0}{(1-g^+(x_{i})-g^-(x_{i}))^{t'-(t-1)}}.
\end{align}
Then for the time-average we have
\begin{align}
    \frac{S_i^t}{t} =  (1-g^+(x_{i})-g^-(x_{i}))^t\frac{S_i^0}{t} + \sum_{t'=0}^{t-1}\frac{g^+(x_{i-1})S_{i-1}^{t'}/t + g^-(x_{i+1})S_{i+1}^{t'}/t + \omega_i^0/t}{(1-g^+(x_{i})-g^-(x_{i}))^{t'-(t-1)}},
\end{align}
which clearly converges when $t\to \infty$ since $\sum_i S_{i}^t = t$, hence, $S_{i+1}^{t'}/t < 1, \;\; S_{i-1}^{t'}/t < 1$, and
\begin{align}
    \lim_{t\to \infty} \frac{S_i^t}{t} = & \lim_{t\to \infty} \sum_{t'=0}^{t-1}\frac{g^+(x_{i-1})S_{i-1}^{t'}/t + g^-(x_{i+1})S_{i+1}^{t'}/t}{(1-g^+(x_{i})-g^-(x_{i}))^{t'-(t-1)}} \leq \\
    & \leq \lim_{t\to \infty} \sum_{t'=0}^{t-1}\frac{g^+(x_{i-1}) + g^-(x_{i+1})}{(1-g^+(x_{i})-g^-(x_{i}))^{t'-(t-1)}}.
\end{align}
Thus, the time-average converges since it is dominated by the convergent series.
To analyse the limit of the time average, we first simplify the notation a little bit by noting that
\begin{align}
    g^+(x) = g(f(x)) \;\text{ and }\; g^-(x) = g(f^{-1}(x)), \;\text{ where }\; g(x) = \frac{p(x)}{q(x)}c.
\end{align}
Then we have $g(x_i) = g^+(x_{i-1}) = g^-(x_{i+1})$, $0 < g(x_i) \leq 1$ and $0 \leq g(x_{i+1}) + g(x_{i-1}) < 1$.
Denoting $A_i^t=S_i^{t}/t$ in \eqref{eq:recurrence_for_sum_diffusing} we write the recurrence relation for the time average
\begin{align}
    A_i^{t+1} = \frac{\omega_i^0}{t} + (1-g(x_{i+1})-g(x_{i-1}))A_i^t +  g(x_{i})(A_{i-1}^{t} + A_{i+1}^t).
\end{align}
Taking the time $t$ large enough we get rid off the initial condition $\omega_i^0$ and consider the stationary point of this recurrence relation.
\begin{align}
    (g(x_{i+1})+g(x_{i-1}))A_i =  g(x_{i})(A_{i-1} + A_{i+1}).
\end{align}
Further, we consider the stationary point $A_i \propto g(x_i)$, and use the fact that the sum of the average across all points equals to $1$: $\sum_i S_{i}^t / t = \sum_i A_{i}^t = 1$, where we take the sum over $i$ for aperiodic orbits $i \in \mathbb{Z}$, and for periodic orbits we take $i$ modulo $T$.
Thus, we have
\begin{align}
    A_{i}^t = \frac{A_{i}^t}{\sum_{j} A_j^t} = \frac{g(x_i)}{\sum_j g(x_j)} = \frac{p(x_i)/q(x_i)}{\sum_j p(x_j)/q(x_j)}.
\end{align}
Using that $q(x_j) \big|\partial f^j/\partial x \big| = q(x_i) \big|\partial f^i/\partial x \big|$, we obtain
\begin{align}
    A_{i}^t = \frac{p(x_i)}{\sum_j p(x_j)(q(x_i)/q(x_j))} = \frac{p(x_i)\big|\partial f^i/\partial x \big|}{\sum_j p(x_j)\big|\partial f^j/\partial x \big|}.
\end{align}
For periodic orbits, the time-average always converges to
\begin{align}
    A_{i}^t = \frac{p(x_i)\big|\partial f^i/\partial x \big|}{\sum_{j=0}^{T-1} p(x_j)\big|\partial f^j/\partial x \big|}.
\end{align}
For aperiodic orbits, if the series $\sum_{j = -\infty}^{+\infty} p(x_j)\big|\partial f^j/\partial x \big|$ converges, then the time-average converges to 
\begin{align}
    A_{i}^t = \frac{p(x_i)\big|\partial f^i/\partial x \big|}{\sum_{j = -\infty}^{+\infty} p(x_j)\big|\partial f^j/\partial x \big|}.
\end{align}
The uniqueness of the limit follows from the mean-ergodic theorem and the stochastic interpretation of the process.
Indeed, in such case the kernel $k(x'\cond x)$ is irreducible on the orbit $\orb(x_0)$.

\section{Applications}
\label{app:examples}

\subsection{Linear combination}
\label{app:linear_combination}
\begin{algorithm}[H]
  \caption{Linear combination of escaping orbital kernels}
  \begin{algorithmic}  
    \INPUT{target density $p(x)$, auxiliary density $p(v\cond x)$ and a sampler from $p(v\cond x)$}
    \INPUT{continuous bijection $f(x,v)$, weights of the combination $\{w_m\}_{m=0}^{T-1}$}
    \STATE initialize $x$
    \FOR{$i = 0\ldots n$}
        \STATE sample $v \sim p(v \cond x)$
        \STATE propose $(x_m,v_m) = f^m(x,v),\; m\in [0,T-1]$
        \STATE evaluate $p(x_m,v_m) \big| \frac{\partial f^m}{\partial [x,v]}\big|$ for all $m \in [0,T-1]$
        \STATE $g_m = \min_{k \in [0,T-1]}\{\frac{p(f^{mk}(x,v))}{p(x,v)} \big| \frac{\partial f^{mk}}{\partial [x,v]}\big|\}\;\; m \in [0,T-1]$ (can be done in parallel)
        \STATE $ x_{m}, w_m \gets
            \begin{cases} 
            x_{m}, w_m , \text{ with probability } g_m\\
            x, w_m , \;\text{ with probability } (1-g_m)
            \end{cases}$
        \STATE $\text{samples} \gets \text{samples} \cup \{(x_m,w_m)\}_{m=0}^{T-1}$
        \STATE $x \gets x_{j}$ with probability $w_j$
    \ENDFOR
    \OUTPUT{samples}
  \end{algorithmic}
  \label{alg:lc_omcmc}
\end{algorithm}
In this section we discuss how one can make use of the linear combination of escaping orbital kernels.
The main benefit of this procedure is that one can evaluate the tests for all powers of $f$ (see Corollary \ref{th:m_step_test}) almost for free.
With more details, consider the set of kernels $k_m(x'\cond x) = \delta(x'-f^m(x))g_m(x) + \delta(x'-x)(1-g_m(x)), \; \forall m \in \mathbb{Z}$, where $g_m(x)$ are tests that guarantee $K_m p= p$.
Then any linear combination of such kernels:
\begin{align}
    k(x'\cond x) = \sum_{m \in \mathbb{Z}}w_m k_m(x'\cond x), \;\; \sum_{m \in \mathbb{Z}} w_m = 1, \;\; w_m \geq 0,
    \label{eq:linear_combination}
\end{align}
clearly admits $Kp = p$.
Considering periodic orbits (period $T$), and taking the tests $g_m(x)=\infimum_{k \in \mathbb{Z}} \{p(f^{mk}(x))/p(x)\big|\partial f^{mk}/\partial x \big| \}$ (as in Corollary \ref{th:m_step_test}) the evaluation of all tests reduces to the measurement of the density (and the determinant of Jacobian) over whole orbit, and then evaluation of minimums across different subsets of points.
Once all the tests are calculated, we can simulate accept/reject steps and then collect all the simulated samples $x_m$ with corresponding weights $w_m$.
The next starting point of the chain can be selected by sampling from the collected samples with probabilities $w_m$.
See Algorithm \ref{alg:lc_omcmc}.

The reversibility criterion from Theorem \ref{th:reversibility} naturally extends to the linear combination as follows.
\begin{proposition}{(Reversibility of the linear combination)\\}
Consider the linear combination of kernels \eqref{eq:linear_combination}, where for each kernel we have $K_m p = p$, and $g_m(x) > 0$.
Then for the linear combination $k(x'\cond x)$ we have $Kp = p$, and it is reversible ($k(x'\cond x) p(x) = k(x\cond x')p(x')$) if and only if $w_{-m}g_{-m}(x) = w_{m}g_{m}(x), \forall m \in \mathbb{Z}$.
Note that for the test from Corollary \ref{th:m_step_test} we have $g_m(x) = g_{-m}(x)$.

For periodic orbit $\orb(x)$ of $f$, assume that the only positive weights are $w_1,\ldots, w_{T-1}$, where $T$ is the period of $\orb(x)$.
Then the linear combination is reversible if and only if $w_{T-m}g_{T-m}(x) = w_m g_{m}(x), \forall m \in \{1,\ldots,T-1\}$.
\end{proposition}
\begin{proof}
See Appendix \ref{app:whole_orbit_reversibility}
\end{proof}

\subsection{Importance sampling via the orbital kernel}
\label{app:oSNIS}

\begin{wrapfigure}{r}{0.35\textwidth}
\vspace{-20pt}
  \begin{center}
    \includegraphics[width=0.32\textwidth]{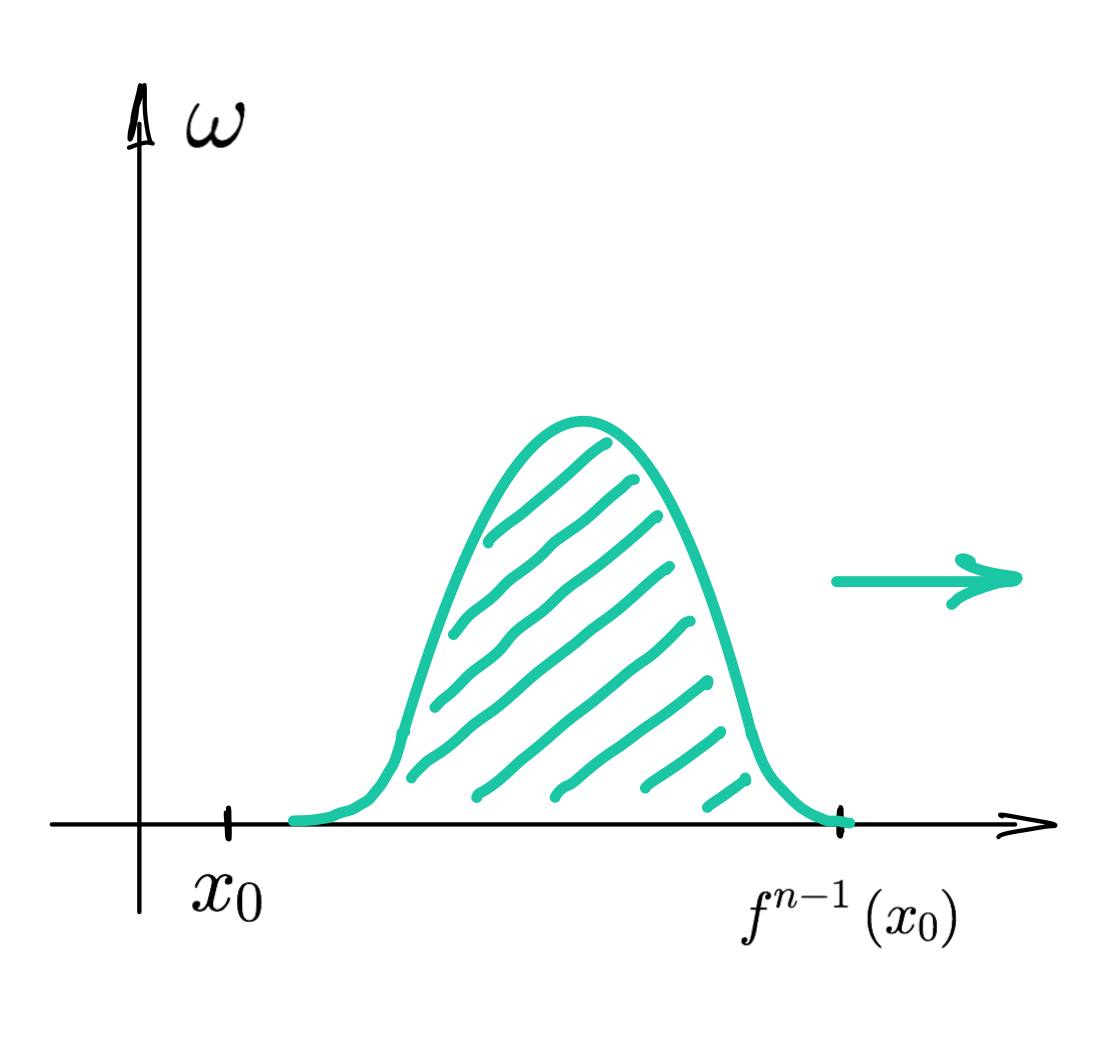}
  \end{center}
  \caption{The illustration of the orbital kernel on an infinite orbit.}
  \label{fig:wave_packet}
\end{wrapfigure}

As we outlined in Section \ref{sec:orbital_kernel}, the limiting behaviour of the orbital kernel depends on the orbit it operates on.
For infinite orbits, the kernel always moves the probability mass further along the orbit escaping any given point on the orbit since there is no mass flowing backward (the weight of $f^{-1}(x_0)$ is zero from the beginning).
Intuitively, we can think that the weights $\omega_i^t$ from \eqref{eq:recurrence_p_t} evolve in time as a wave packet moving on the real line (Fig. \ref{fig:wave_packet}).
With this intuition it becomes evident that the time average at every single point on the orbit converges to zero since the mass always escapes this point.
Therefore, if we want to accept several points from the orbit, we need to wait until the kernel ``leaves'' this set of points and then stop the procedure.
To avoid the explicit simulation of the kernel, we provide an approximation for the escape time in the following proposition.

\begin{proposition}{(Average escape time)\\}
\label{th:escape_time}
Consider the proper escaping orbital kernel ($Kp = p$, and $g(x) > 0$), and the initial distribution $p_0(x) = \delta(x-x_0)$.
The escape time $t_n$ is the number of iterations required to leave the set $\{x_0,f(x_0),\ldots,f^{n-1}(x_0)\}$, or equivalently the time of the first acceptance of $f^n(x_0)$. 
The expectation of $t_n$ is $\mean t_n = \sum_{i=0}^{n-1} 1/g(f^i(x_0))$.
\end{proposition}
\begin{proof}
See Appendix \ref{app:escape_time}.
\end{proof}
\vspace{-10pt}
Using this proposition, we approximate the time average for points $\{x_0,\ldots,f^{n-1}(x_0)\}$ as follows.
Once the kernel escaped $f^{n-1}(x_0)$, the sum of each individual weight over time has converged.
Moreover, from Theorem \ref{th:limiting_case_escaping} we know $\lim_{t\to \infty} \sum_{t'=0}^t \omega_i^{t'} = 1/g(f^i(x_0))$.
To approximate the time-average we can divide this sum by the average escape time from Proposition \ref{th:escape_time}.
Thus, we estimate $p_{n-1}(x) \approx \sum_{i=0}^{n-1} \omega_i \delta(x-f^i(x_0))$, where
\begin{align}
    \begin{split}
    \omega_i =\frac{1}{\mean t_n}\lim_{t\to \infty} \sum_{t'=0}^t \omega_i^{t'} = \frac{1/g(f^i(x_0))}{\sum_{j=0}^{n-1} 1/g(f^j(x_0))} = \frac{p(f^i(x_0))\big|\frac{\partial f^{i}}{\partial x}\big|_{x=x_0}}{\sum_{j=0}^{n-1}p(f^j(x_0))\big|\frac{\partial f^{j}}{\partial x}\big|_{x=x_0}}.
    \end{split}
    \label{eq:limiting_approx}
\end{align}
On the last step, the ratio $g(f^i(x_0))/g(f^j(x_0))$ can be found from \eqref{eq:preserving_gp}, i.e., using the fact that the kernel $K$ preserves the target density.
The resulting formula is tightly related to the self-normalized importance sampling (SNIS) \citep{andrieu2003introduction}.
Indeed, if $f$ preserves some density $q$ on the orbit $\orb(x_0)$, then we can rewrite the Jacobians using this density and put it into \eqref{eq:limiting_approx} as follows.
\begin{align}
    \forall i \;\;\; \bigg|\frac{\partial f^{i}}{\partial x}\bigg|_{x=x_0} = \frac{q(x_0)}{ q(f^i(x_0))}; \text{ hence, }\omega_i = \frac{p(f^i(x_0))/q(f^i(x_0))}{\sum_{j=0}^{n-1}p(f^j(x_0))/q(f^j(x_0))}
\end{align}
\begin{wrapfigure}{r}{0.4\textwidth}
\vspace{-0pt}
  \begin{center}
    \includegraphics[width=0.39\textwidth]{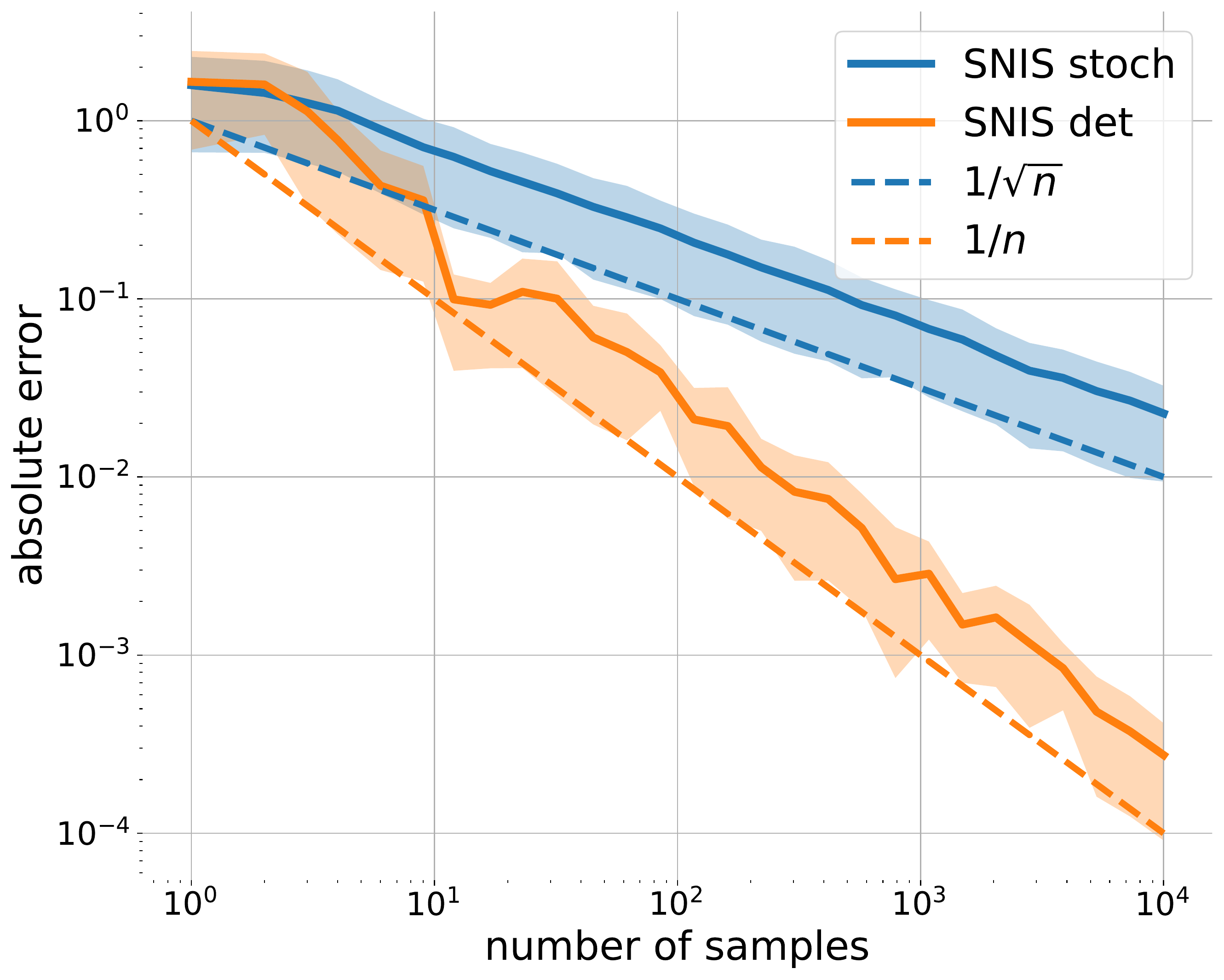}
  \end{center}
  \caption{The comparison of deterministic and stochastic SNIS.
  We average across $200$ independent runs, starting from random points for deterministic SNIS.
  Solid lines demonstrate the mean error of the estimate, and the shaded area lies between $0.25$ and $0.75$ quantiles of the error. Dashed lines demonstrate the asymptotic of estimates.}
  \label{fig:snis}
\vspace{-20pt}
\end{wrapfigure}

We illustrate the usefulness of the formula \eqref{eq:limiting_approx} with the following example.
Consider the target distribution $p$, and the proposal $q$.
Then one can estimate the mean of the target using the conventional SNIS procedure, which we call stochastic SNIS.
Stochastic SNIS operates as follows. It samples from the proposal distribution with density $q(x)$, and then evaluates the weight of sample $x_i$ as $w_i = p(x_i)/q(x_i)$, once all the weights are evaluated it normalizes them as $w_i \gets w_i / \sum_j w_j$.
The normalization of the weights is usually done in the setting where the target density is known up to the normalization constant.
Another way to collect samples is to consider a deterministic map that preserves $q$ and then re-weight the trajectory of this map (we call it deterministic SNIS).
To perform the deterministic SNIS we choose such $f$ that (I) preserves the density of the proposal $q(x)$ (the same as for stochastic SNIS), and (II) has dense aperiodic orbits on the state space.
These two goals are achieved by the map described in \citep{murray2012driving}.
Namely, the deterministic map is $f(x) = F^{-1}((F(x) + a) \text{ mod } 1)$, where $a$ is the irrational number (to cover densely the state space) and $F$ is the CDF of $q$ (to preserve $q$).
Another way of thinking about iterations of $f(x)$ is that it firstly sample from $\text{Uniform}[0,1]$ using the Weyl's sequence $u_{i+1} = u_i + a \text{ mod } 1$, and then use the inverse CDF to obtain samples from the desired distribution.

For the target distribution we take the mixture of two Gaussians $p(x) = 0.5\cdot\mathcal{N}(x\cond -2,1) + 0.5\cdot\mathcal{N}(x\cond 2,1)$, and for the proposal we take single Gaussian $q(x)=\mathcal{N}(x\cond 0,2)$.
For the irrational number $a$ in the deterministic SNIS, we take the float approximation of $\sqrt{2}$.
Collecting samples from both procedures we estimate the mean of the target distribution and evaluate the squared error of the estimation.
Fig.~\ref{fig:snis} demonstrates that the deterministic SNIS allows for a more accurate estimate having the same number of samples.

The practical benefit of the formula \eqref{eq:limiting_approx} comes from the fact that we can perform deterministic SNIS even when we don't know the stationary distribution $q$ of the deterministic map.

\subsection{Orbital MCMC with Hamiltonian dynamics}
\label{app:oHMC}

Building upon Algorithms \ref{alg:omcmc} and \ref{alg:omcmc_contracting}, we consider special cases, which use the Hamiltonian dynamics for the deterministic map $f(x,v)$.
That is, we consider the joint distribution $p(x,v) = p(x)\mathcal{N}(v\cond 0,\mathbbm{1})$, and for the map $f(x,v)$ we take the Leapfrog integrator ($f(x,v) = [x',v'']$):
\begin{align}
\begin{split}
    &v' = v - \frac{\eps}{2} \nabla_x (-\log p(x)),\\
    &x' = x + \eps \nabla_v (-\log p(v')),\\
    &v'' = v' - \frac{\eps}{2} \nabla_x (-\log p(x')),
\end{split}
\end{align}
where $\eps$ is the step-size. Adding the directional variable as proposed in Algorithm \ref{alg:omcmc} we obtain the Orbital-HMC algorithm. Note that the Jacobian of the map is $\big|\partial f(x,v)/\partial [x,v]\big| = 1$.

For the contractive map in Algorithm \ref{alg:omcmc_contracting}, we consider the same Hamiltonian dynamics but with friction. To add friction, we follow \citep{francca2020conformal}, but use the velocity Verlet integrator instead of the coordinate. The resulting integrator is ($f(x,v) = [x',v'']$):
\begin{align}
\begin{split}
    &v' = \beta\big[v - \frac{\eps}{2} \nabla_x (-\log p(x))\big],\\
    &x' = x + \frac{\eps}{2} (\beta^{-1} + \beta) \nabla_v (-\log p(v')),\\
    &v'' = \beta\big[v' - \frac{\eps}{2} \nabla_x (-\log p(x'))\big],
\end{split}
\end{align}
where $\beta$ is the contractive coefficient. Note that $\beta = 1$ yields the standard Leapfrog. The Jacobian of the map is $\big|\partial f(x,v)/\partial [x,v]\big| = \beta^{2n}$, where $n$ is the number of dimensions of the target distribution.
For $\beta < 1$, the map becomes contractive with the stationary points at $\nabla_x \log p(x) = 0$.
We refer to the resulting algorithm as Opt-HMC due to its optimization properties.

To tune the hyperparameters for all algorithms we use the ChEES criterion \citep{hoffman2021adaptive}. During the initial period of adaptation, this criterion optimizes the maximum trajectory length $T_{\text{max}}$ for the HMC with jitter (trajectory length at each iteration is sampled $\sim \text{Uniform}(0,T_{\text{max}})$). To set the stepsize of HMC we follow the common practice of keeping the acceptance rate around $0.65$ as suggested in \citep{beskos2013optimal}.
We set this stepsize via double averaging as proposed in \citep{hoffman2014no} and considered in ChEES-HMC.
For Opt-HMC, we don't need to set the trajectory length, but we use the step size yielded at the adaptation step of ChEES-HMC. The crucial hyperparameter for this algorithm is the friction coefficient $\beta$, which we set to $\sqrt[n]{0.8}$, where $n$ is the number of dimensions of the target density, thus setting the contraction rate to $0.64$.

In Figures \ref{fig:app_mean_error} and \ref{fig:app_std_error}, we compare the errors in the estimation of the mean and the variance of the target distribution as a function of the number of gradient evaluations (which we take as a hardware-agnostic estimation of computation efforts).
Opt-HMC provides the best estimate for the mean value for all distributions except for the ill-conditioned Gaussian (where Orbital-HMC demonstrates the fastest convergence).
Another downside of Opt-HMC is that it relatively poorly estimates the variance of the target as provided in Fig. \ref{fig:app_std_error}, which could also be explained by the introduced contraction of space.
Note that Orbital-HMC always performs comparably or better than competitors.

\begin{figure}[h]
    \centering
    \includegraphics[width=0.24\textwidth]{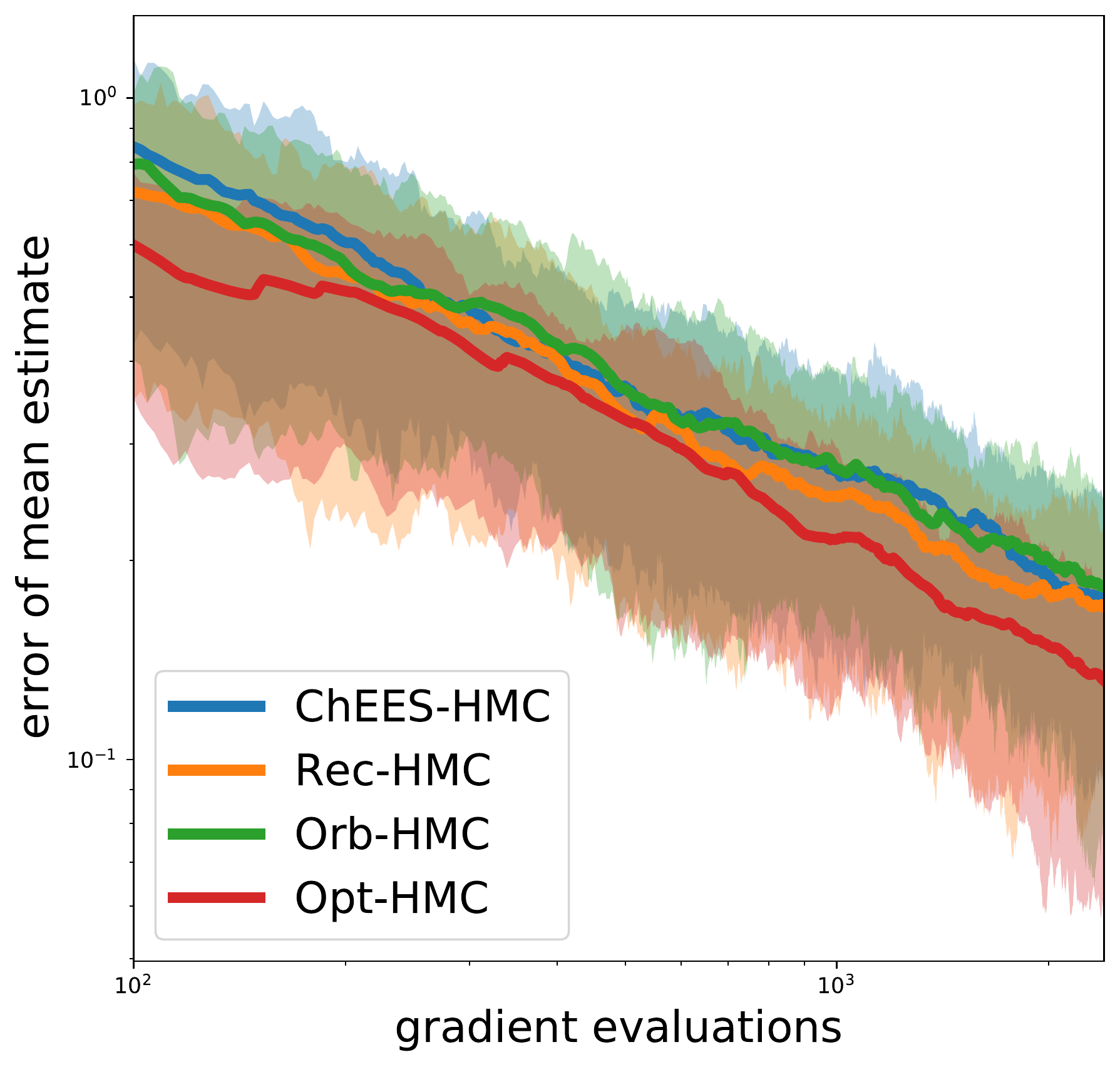}
    \includegraphics[width=0.24\textwidth]{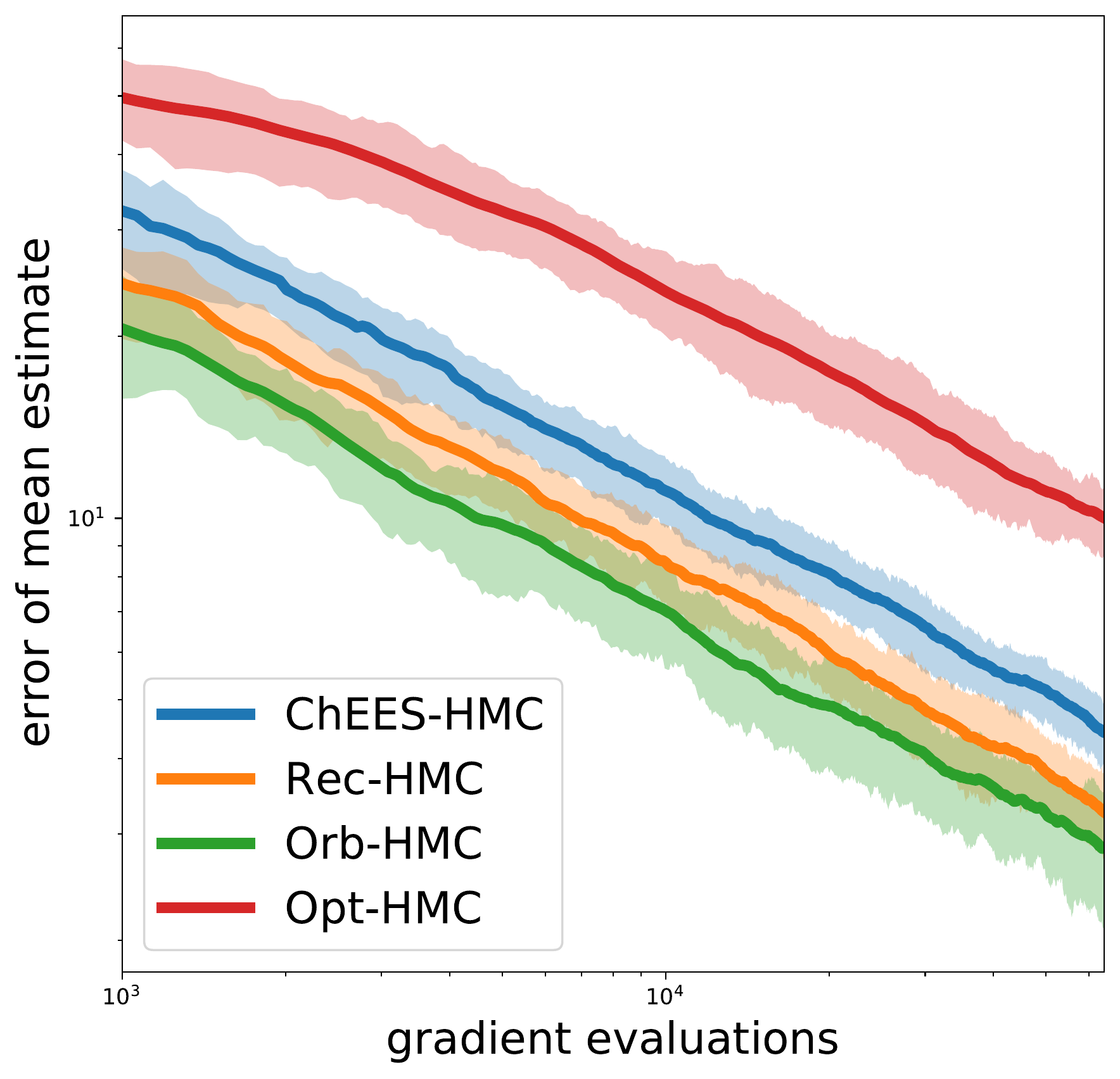}
    \includegraphics[width=0.24\textwidth]{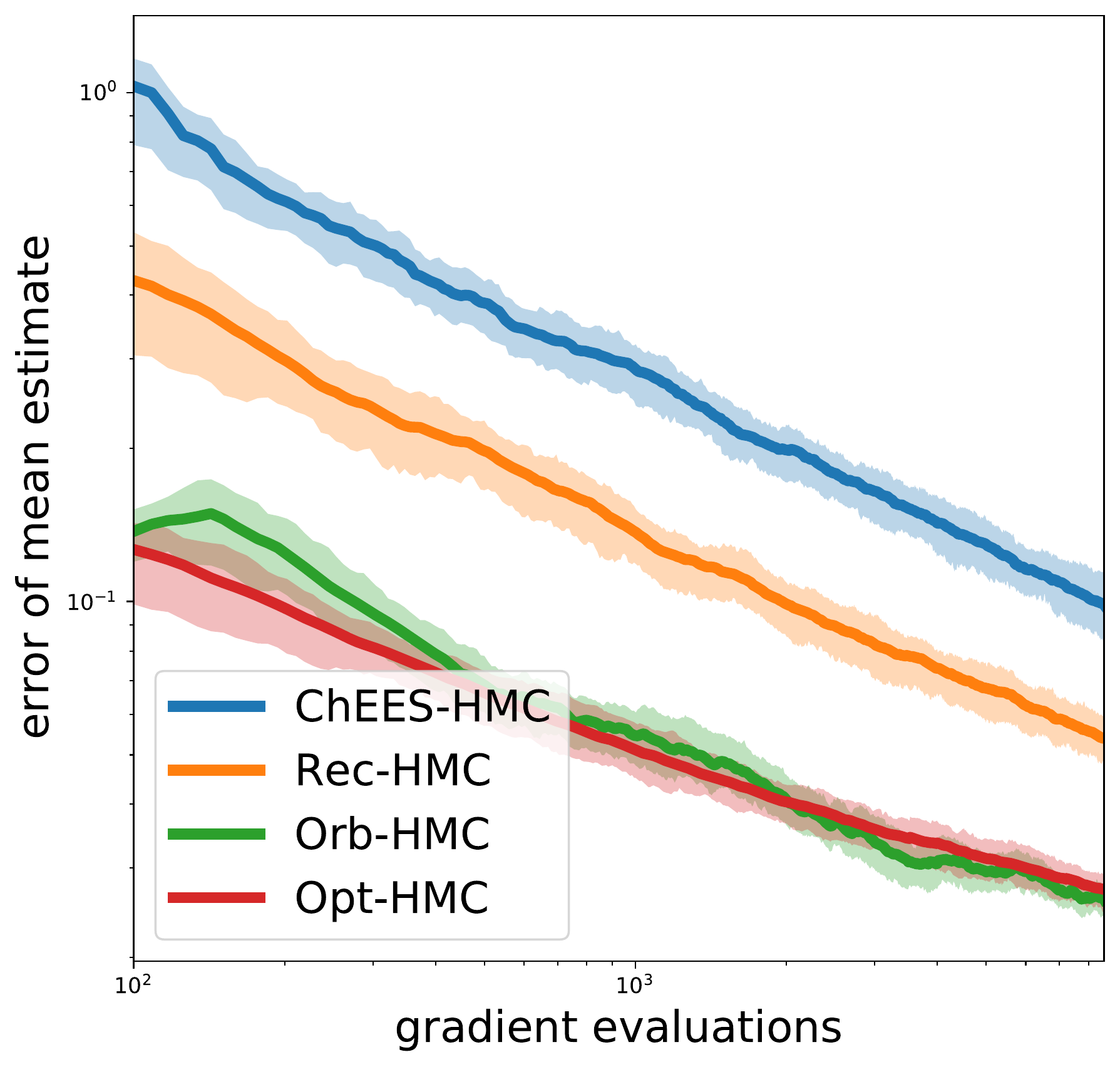}
    \includegraphics[width=0.24\textwidth]{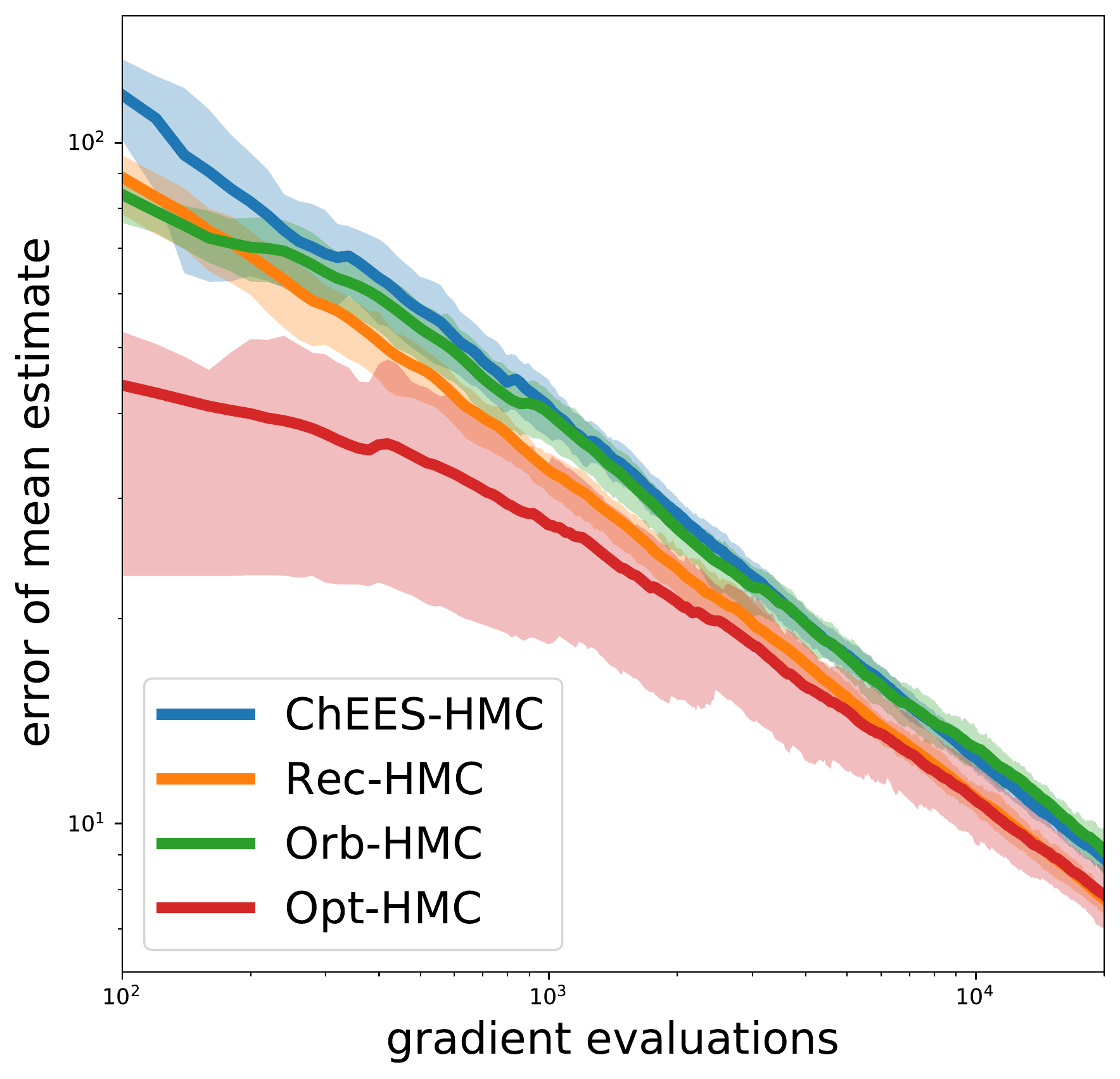}
    \caption{From left to right: the error of mean estimation on Banana, ill-conditioned Gaussian, logistic regression, Item-Response model. Every solid line depicts the mean of the absolute error averaged across $100$ independent chains. The shaded area lies between $0.25$ and $0.75$ quantiles of the error.}
    \label{fig:app_mean_error}
\end{figure}
\begin{figure}[h]
    \centering
    \includegraphics[width=0.24\textwidth]{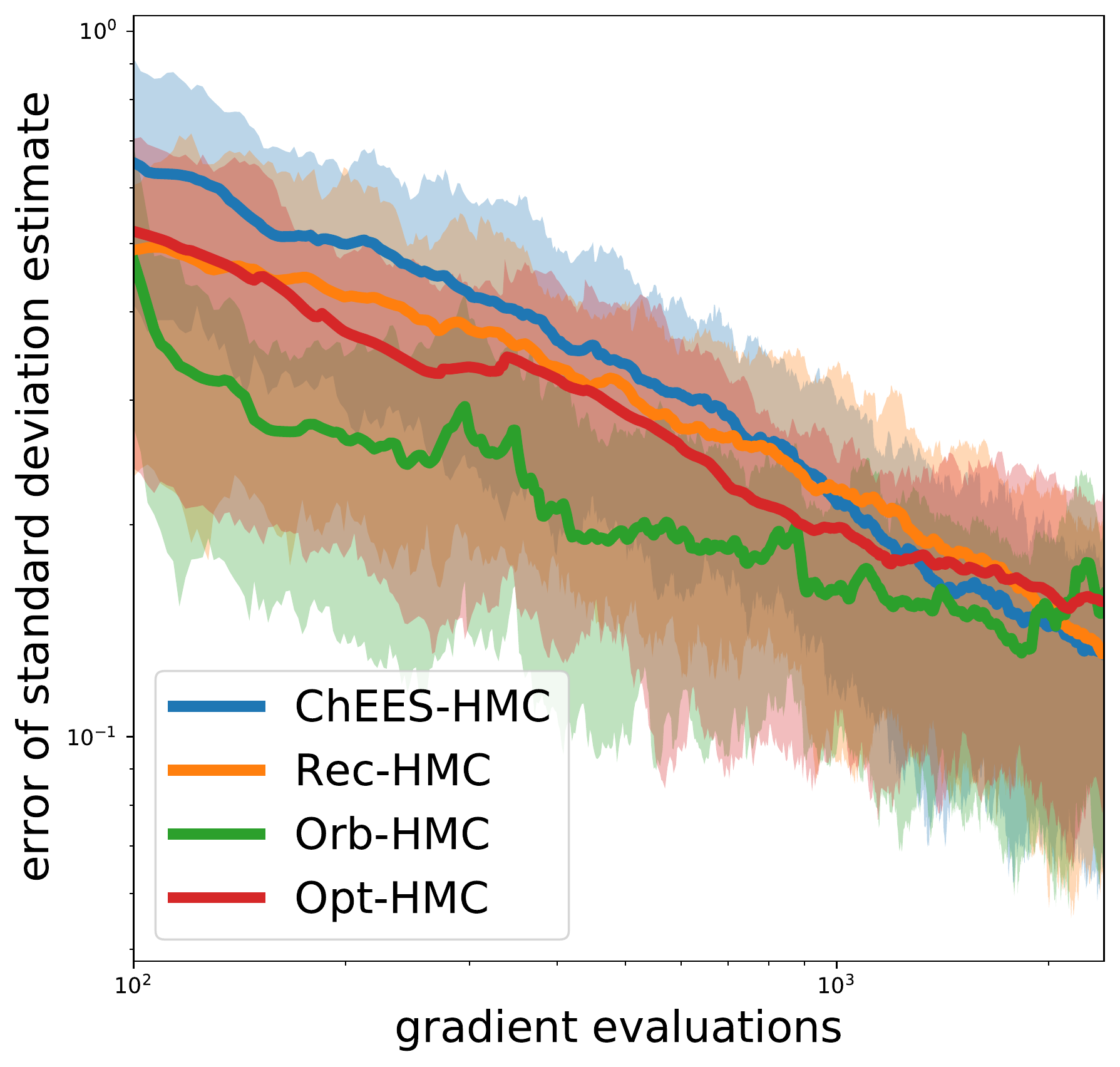}
    \includegraphics[width=0.24\textwidth]{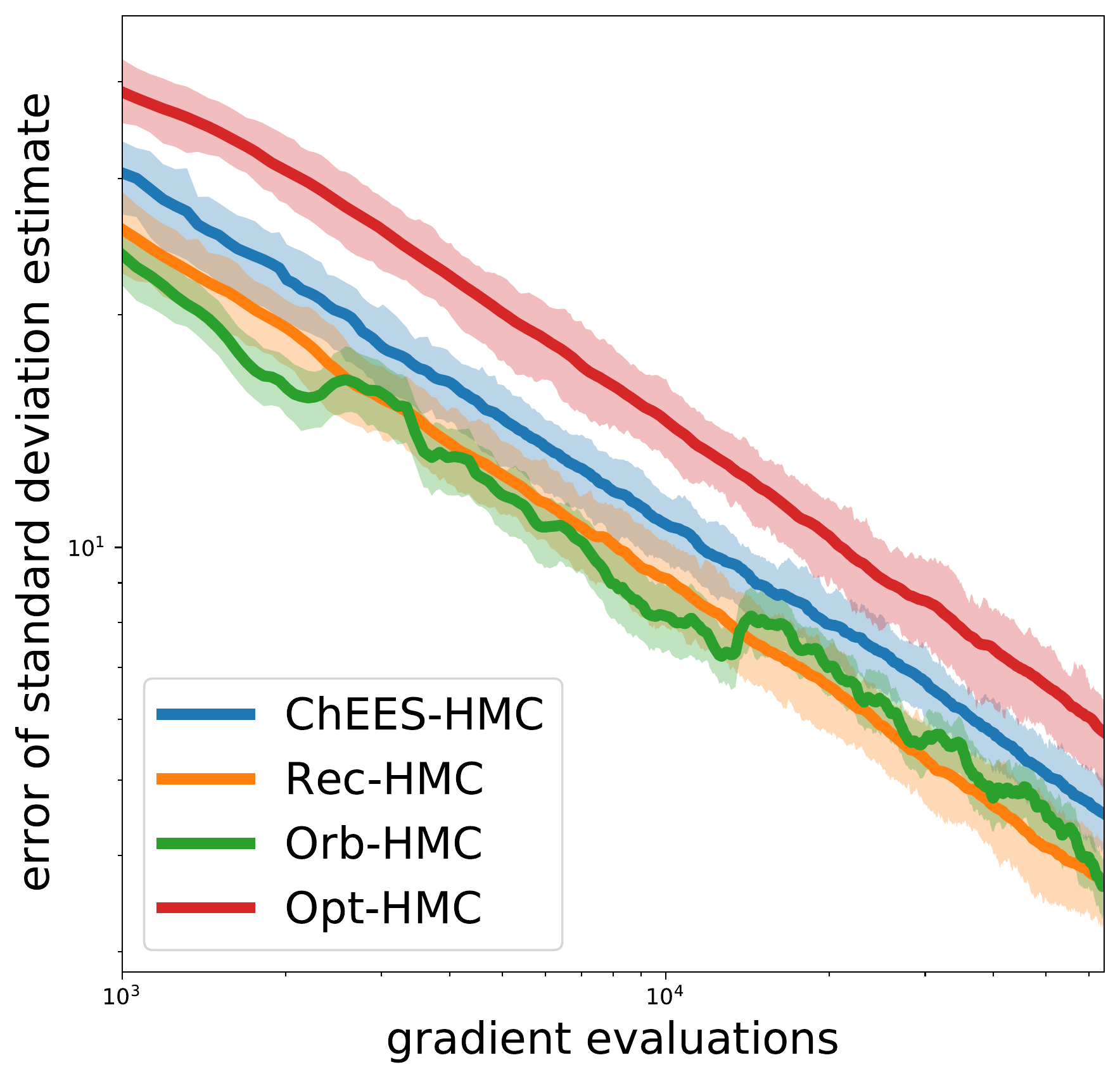}
    \includegraphics[width=0.24\textwidth]{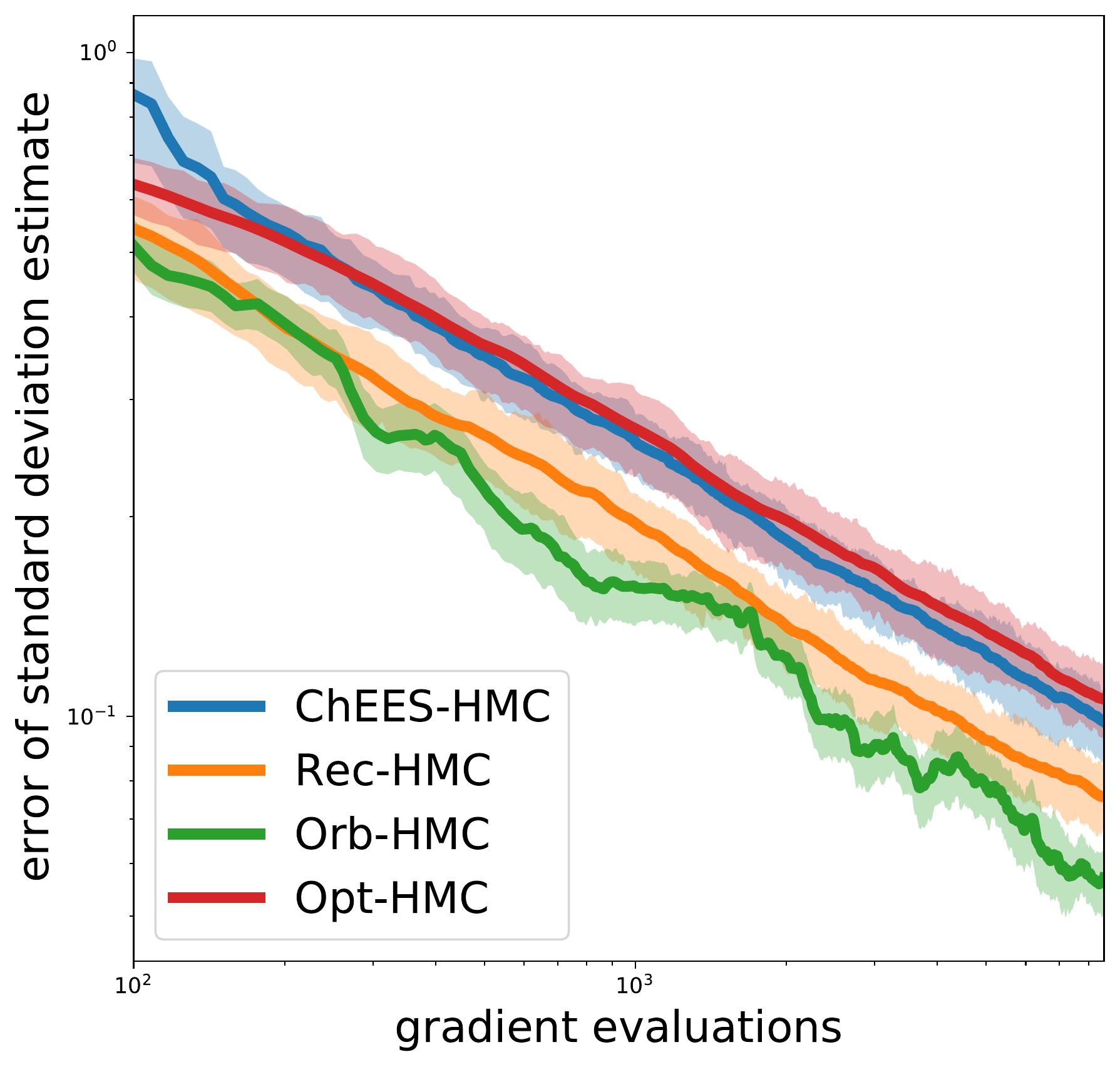}
    \includegraphics[width=0.24\textwidth]{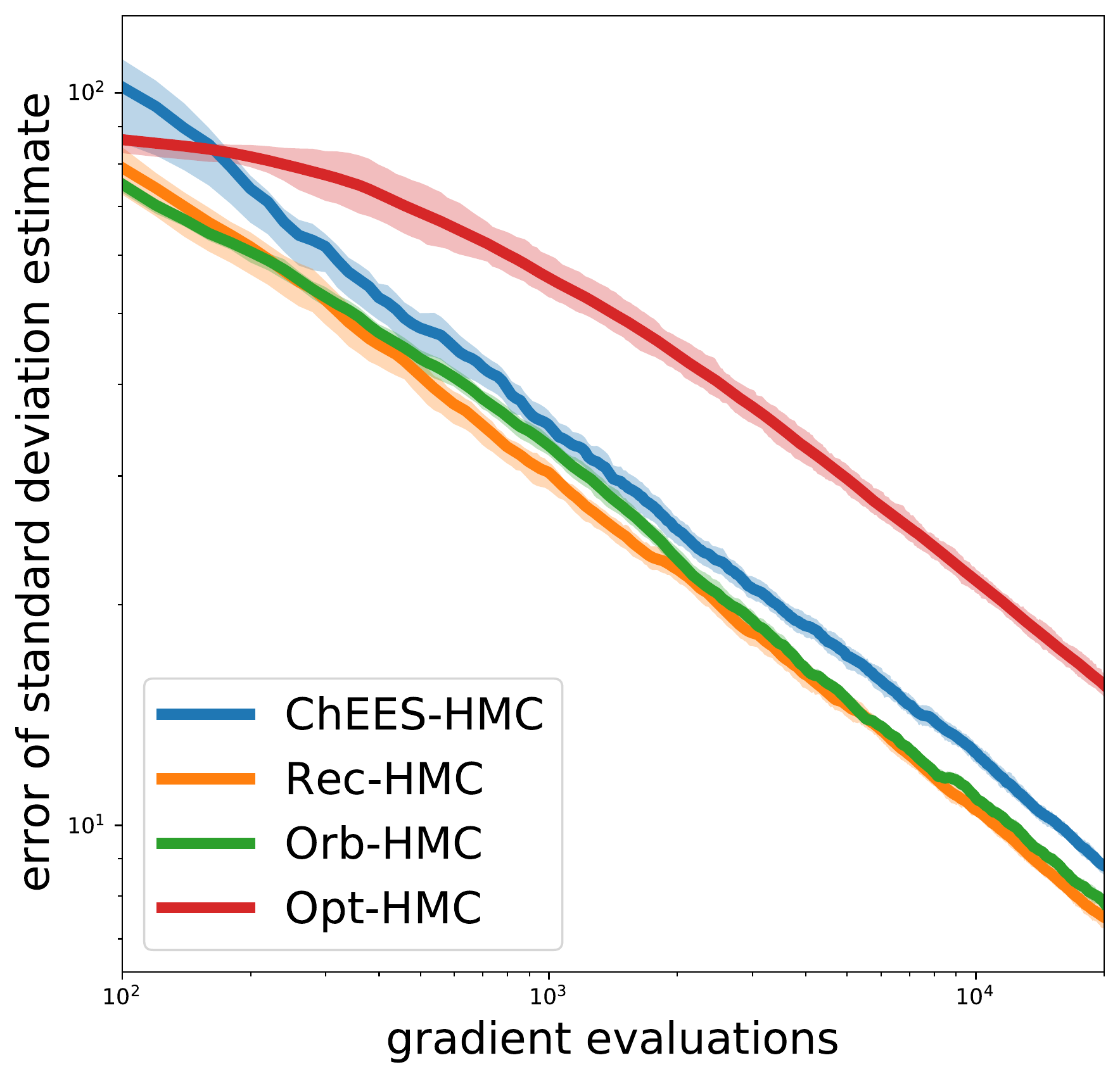}
    \caption{From left to right: the error of the variance estimation on Banana, ill-conditioned Gaussian, logistic regression, Item-Response model. Every solid line depicts the mean of the absolute error averaged across $100$ independent chains. The shaded area lies between $0.25$ and $0.75$ quantiles of the error.}
    \label{fig:app_std_error}
\end{figure}

\subsection{Distributions}
\label{app:dists}

Here we provide analytical forms of considered target distributions.

\textbf{Banana (2-D):}
\begin{equation}
p(x_1, x_2) = \Normal(x_1 | 0, 10)\Normal(x_2 | 0.03(x_1^2-100), 1)
\end{equation}

\textbf{Ill-conditioned Gaussian (50-D):}
\begin{equation}
p(x) = \Normal(x|0, \Sigma),
\end{equation}
where $\Sigma$ is diagonal with variances from $10^{-2}$ to $10^{2}$ in the log scale.

\textbf{Bayesian logistic regression (25-D):}\\
For the Bayesian logistic regression, we define likelihood and prior as
\begin{equation}
    p(y=1\cond x, \theta) = \frac{1}{1+\exp(-x^T\theta_w+\theta_b)}, \;\;\; p(\theta) = \Normal(\theta\cond 0, \mathbbm{1}).
\end{equation}
Then the unnormalized density of the posterior distribution for a dataset $D=\{(x_i,y_i)\}_i$ is
\begin{equation}
    p(\theta\cond D) \propto \prod_i p(y_i\cond x_i, \theta) p(\theta).
\end{equation}
We use the German dataset ($25$ covariates, $1000$ data points).

\textbf{Item-response theory (501-D):}\\
The model is defined by the joint distribution:
\begin{align}
    p(y, \alpha, \beta, \delta) = & \prod_{n=1}^N\text{Bernoulli}(y_n\cond\sigma(\alpha_{j_n} - \beta_{k_n} + \delta)) \cdot \\
    &\cdot \prod_j\Normal(\alpha_j\cond0,1) \prod_k \Normal(\beta_k\cond 0,1) \Normal(\delta\cond 0.75,1).
\end{align}
Here $\alpha_j$ could be interpreted as the skill level of the student $j$; $\beta_k$ is the difficulty of the question $k$; then $y_n$ is the answer of a student to a question. We consider $100$ students, $400$ questions and $30105$ responses. We generate the data from the prior.

\subsection{Usage of neural models}
\label{app:neural_models}

In this section we discuss how one can possibly use the expressive learnable models (such as flows \citep{rezende2015variational, dinh2016density}) together with the orbital kernel to design an efficient sampler.
As we discuss in Section \ref{sec:practice} and Appendix \ref{app:oHMC}, to design an unbiased kernel one needs to design the periodic function.
Therefore, in the next two subsection we consider the design of periodic functions using the family of normalizing flows.

\subsubsection{Topologically conjugate sampler}
This section operates similarly to \citep{hoffman2019neutra}.
That is, having some simple function $f$ and a learnable continuous bijection $T$ one can sample using the topologically conjugate iterated function $h = T^{-1}\circ f \circ T$.
Indeed, if the evaluation of $f$ is cheap, then we can evaluate iterations of $h$ as
\begin{align}
    h^n = T^{-1}\circ f^n \circ T.
\end{align}
For instance, we consider a simple periodic function $f$ assuming that all the knowledge about the target distribution is encapsulated in $T$.
Namely, we take Hamiltonian dynamics that can be integrated exactly:
\begin{align}
    H(x,v) = & \frac{1}{2}x^Tx + \frac{1}{2}v^Tv, \\
    x_i(\tau) = & x_i(0)\cos(\tau) + v_i(0)\sin(\tau),\\
    v_i(\tau) = & -x_i(0)\sin(\tau) + v_i(0)\cos(\tau).
\end{align}
These equations define a continuous flow $f(x,\tau)$.
To obtain discrete rotations, we can fix the evolution time either as $\tau=2\pi\frac{1}{T}$ (obtaining periodic orbits with period $T$), or $\tau = 2\pi a$, where $a$ is irrational number (obtaining returning orbits).

\subsubsection{Periodic flows}
Instead of the learning of the target space embedding, as in the previous section, we can learn the deterministic function $f$ itself.
For that purpose, we firstly introduce \textit{periodic coupling layer}.
Similarly to the coupling layer from \citep{dinh2016density}, we define the periodic coupling layer $l(x,\tau): \mathbb{R}^D \to \mathbb{R}^D$ as follows.
\begin{align}
    y_{1:d} =& x_{1:d} \\
    y_{d+1:D} =& x_{d+1:D} \odot \exp\bigg(\sin(\omega\tau) s(x_{1:d})\bigg) + \sin(\omega\tau) h(x_{1:d})
\end{align}
Here $s(\cdot)$ and $h(\cdot)$ are neural networks that define scale and shift respectively; $\omega$ is the circular frequency, which is a hyperparameter.
Stacking these layers we obtain the \textit{periodic flow} $f(x,\tau)$, which is invertible for fixed $\tau$ and has tractable Jacobian.
To guarantee periodicity, one can choose $\omega_i$ for each layer $l_i$ as a natural number.
Thus, we guarantee that $f(x,0) = x$ and $f(x,2\pi n) = x$ for natural $n$.

However, such flow does not satisfy $f(x,\tau_1+\tau_2) = f(f(x,\tau_1),\tau_2)$.
To ensure this property, we extend the state space by the auxiliary variable $\tau$ and introduce function $\widehat{f}((x,\tau),\Delta\tau) = [x',\tau']$ on the extended space:
\begin{align}
    \tau' = (\tau+\Delta\tau) \; \text{mod} \; 2\pi, \;\;
    x' = f\bigg(f^{-1}(x,\tau), \tau'\bigg),
\end{align}
where the inversion of $f$ is performed w.r.t. the $x$-argument.
By the straightforward evaluation, we have
\begin{align}
    \widehat{f}((x,\tau),0) = \widehat{f}((x,\tau),2\pi) = (x,\tau), \;\; 
    \widehat{f}\bigg(\widehat{f}((x,\tau),\Delta\tau_1),\Delta\tau_2\bigg) = \widehat{f}((x,\tau),\Delta\tau_1+\Delta\tau_2).
\end{align}
Finally, as in the previous section, we can obtain discrete rotations by considering $\Delta\tau = 2\pi\frac{1}{n}$ or $\Delta\tau = 2\pi a$ for irrational $a$.


\end{document}